\documentclass[runningheads]{llncs}

\usepackage[T1]{fontenc}
\def\doi#1{\href{https://doi.org/\detokenize{#1}}{\url{https://doi.org/\detokenize{#1}}}}
\usepackage{graphicx}
\usepackage{listings}
\usepackage[square,numbers]{natbib}
\usepackage{booktabs}
\usepackage{multirow}

\usepackage{amsfonts}
\usepackage{amsmath}

\usepackage[ruled,linesnumbered]{algorithm2e}

\usepackage{xspace}

\usepackage{hyperref}
\usepackage{cleveref}

\usepackage{etoolbox}
\providetoggle{arxiv}

\usepackage{aircraftshapes}
\usepackage{here}

\renewcommand{\vec}[1]{\boldsymbol{\mathbf{#1}}}

\DeclareMathOperator*{\argmin}{arg\,min}

\newcommand{\fsat}{{f_{\mathrm{sat}}}}

\usepackage{wasysym}
\usepackage[clock]{ifsym}

\newcommand{\tool}{SpecRepair\xspace}

\newcommand{\exfailure}{\ensuremath{\times}}
\newcommand{\extimeout}{\bell}
\newcommand{\exsuccess}{\ensuremath{\checkmark}}
\newcommand{\exunknown}{?}

\usepackage{subcaption}
\usepackage{caption}
\usepackage{tabularx}
\usepackage{longtable}
\usepackage{tikz}
\usepackage{pgfplots}
\usepgfplotslibrary{groupplots}
\usepgfplotslibrary{fillbetween}

\pgfplotsset{width=4.5cm,compat=1.9,height=3.7cm}

\definecolor{Colors-A}{RGB}{228,26,28}  
\definecolor{Colors-C}{RGB}{55,126,184}  
\definecolor{Colors-B}{RGB}{77,175,124}  
\definecolor{Colors-D}{RGB}{122,48,133}  
\definecolor{Colors-E}{RGB}{255,127,0}  
\definecolor{Colors-F}{RGB}{207,129,191}  

\newcommand{\mat}[1]{\ensuremath{#1}\xspace}

\definecolor{Set1-A}{RGB}{228,26,28}

\settoggle{arxiv}{true}

\begin{document}
\title{\tool: Counter-Example Guided Safety Repair of Deep Neural Networks}
\titlerunning{\tool: Counter-Example Guided Repair of Neural Networks}

%
%
\author{Fabian Bauer-Marquart\inst{1} \and
David Boetius\inst{1} \and
Stefan Leue\inst{1} \and \\
Christian Schilling\inst{2}}
\authorrunning{F. Bauer-Marquart et al.}
%
\institute{University of Konstanz, Germany\\
\and
Aalborg University, Denmark}
\maketitle              

\begin{abstract}
Deep neural networks (DNNs) are increasingly applied in safety-critical domains, such as self-driving cars, unmanned aircraft, and medical diagnosis.
It is of fundamental importance to certify the safety of these DNNs, i.e. that they comply with a formal safety specification.
While safety certification tools exactly answer this question, they are of no help in debugging unsafe DNNs, requiring the developer to iteratively verify and modify the DNN until safety is eventually achieved.
Hence, a repair technique needs to be developed that can produce a safe DNN automatically.
To address this need, we present SpecRepair, a tool that efficiently eliminates counter-examples from a DNN and produces a provably safe DNN without harming its classification accuracy.
SpecRepair combines specification-based counter-example search and resumes training of the DNN, penalizing counter-examples and certifying the resulting DNN.
We evaluate SpecRepair's effectiveness on the ACAS Xu benchmark, a DNN-based controller for unmanned aircraft, and two image classification benchmarks.
The results show that SpecRepair is more successful in producing safe DNNs than comparable methods, has a shorter runtime, and produces safe DNNs while preserving their classification accuracy.

\end{abstract}

\section{Introduction}

Autonomous systems are increasingly steered by machine-learned controllers.
The moment these controllers are integrated into self-driving cars \cite{DBLP:conf/ijcnn/OnishiMSMO19}, unmanned drones \cite{julian2016policy}, or software for medical diagnosis \cite{DBLP:journals/itpro/DjavanshirCY21}, they become safety-critical. 
Machine learning models, such as \textit{deep neural networks} (DNNs), have been shown to not be robust against small modifications to the input \cite{DBLP:journals/corr/SzegedyZSBEGF13}.
These inputs, which we call \emph{counter-examples}, can radically change the classification outcome, thus leading to safety hazards.
Consequently, various safety certification tools have been proposed to show the absence of counter-examples \cite{DBLP:journals/csr/HuangKRSSTWY20}.
However, if the controller is not safe, using these tools results in a tedious process of iteratively verifying and modifying the controller until a safe version, free of any counter-examples, is eventually obtained.
While several methods have addressed this problem by focusing only on classification robustness \cite{DBLP:journals/corr/GoodfellowSS14,DBLP:conf/iclr/MadryMSTV18}, it is essential to target the more general \emph{formal safety properties}~\cite{DBLP:journals/tse/Lamport77} instead.
Such logic properties are crucial when analyzing and verifying these safety-critical systems.

To address the issues mentioned above, we introduce \emph{SpecRepair}, a safety repair tool for DNNs  that renders manual iterative modification obsolete.
This is achieved by a counter-example search algorithm tailored to formal safety properties, a repair procedure that balances accuracy and counter-example elimination, and a final safety certification.

\paragraph{Related Work.}\label{par:related-work}

We summarize three threads of work towards counter-example search and repair of DNNs in the context of formal safety specifications:

\textbf{Certification.}
The verification community has developed techniques that provably determine a DNN's safety \cite{DBLP:journals/ftopt/LiuALSBK21}: either giving a formal guarantee that the specification is satisfied or finding a counter-example.
This problem is NP-hard, and approaches such as \cite{DBLP:conf/cav/HuangKWW17,DBLP:conf/cav/KatzHIJLLSTWZDK19} solve the problem precisely and thus are only suitable for relatively small DNNs.
ERAN \cite{DBLP:journals/pacmpl/SinghGPV19} uses an abstract interpretation to make verification more scalable. 
To find a counter-example, all the above approaches require a logic encoding of the DNN to use SMT or MILP solvers, which limits their scalability, whereas we found that our way to obtain counter-examples even scales to large DNNs.

\textbf{Adversarial attacks and adversarial search.}
The machine-learning community has designed several algorithms to find counter-examples. 
Most algorithms only consider \emph{classification robustness}, which expresses that a classifier assigns the same label to similar inputs.
\citet{DBLP:journals/corr/GoodfellowSS14} proposed the fast-gradient sign method (FGSM), one of the first such attack algorithms.
\citet{DBLP:conf/icml/MoonAS19} accelerate search via an optimization procedure for image classifiers that perturbs only parts of the input image, and \citet{DBLP:conf/sp/ChenJW20} generate counter-examples optimized for the $\ell_2$ and $\ell_\infty$ norms.
DL2 by \citet{DBLP:conf/icml/FischerBDGZV19} can express specifications beyond robustness and send queries to a basin-hopping optimizer.
Its repair capabilities are discussed in the next paragraph.
Some works monitor a DNN for adversarial attacks but do not target other safety properties~\cite{HenzingerLS20,LukinaSH21,Cheng21}.

\textbf{Formal safety repair of neural networks.}
Verification and adversarial attacks alone only analyze DNNs statically.
The ultimate goal, however, is to repair the DNNs such that they become \emph{provably safe} and at the same time \emph{maintain high classification accuracy} (the ratio of correct classifications amongst all inputs).
We explicitly distinguish the concept of formal safety repair from `repairs' that mainly target improving a model's test accuracy, as done in \cite{DBLP:conf/cav/UsmanGSNP21}.
We also need to distinguish formal safety repair from adversarial defence techniques, such as \cite{DBLP:conf/iclr/MadryMSTV18,DBLP:journals/cacm/GoodfellowPMXWO20}, as these do not lead to any guarantees.
DL2 by \citet{DBLP:conf/icml/FischerBDGZV19} integrates logic constraints into the DNN training procedure, but does not give formal guarantees that the resulting DNN ultimately satisfies these logic constraints.
\citet{DBLP:conf/lpar/GoldbergerKAK20} (minimal modification) use the verifier Marabou \cite{DBLP:conf/cav/KatzHIJLLSTWZDK19} to directly modify network weight parameters to satisfy a given specification; the technique is based on SMT solving and hence suffers from limited scalability.
Also, the modification of such parameters may harm the DNN's accuracy.
The approach nRepair by \citet{DBLP:conf/qrs/DongSWWD21} iteratively generates counter-examples using a verifier. Instead of modifying the DNN directly, violating inputs are sent to a copy of the original DNN with modified parameters. Then, the combined model is verified again until no counter-example is found. The method only handles fully-connected feed-forward DNNs and does not support convolutional neural networks (CNNs). In the evaluation, we show that our approach is more efficient.
Finally, \citet{DBLP:conf/pldi/SotoudehT21}, similarly to \cite{DBLP:conf/lpar/GoldbergerKAK20}, aim to minimally modify a given DNN according to a formal specification using an LP solver. However, the specifications that the method supports are limited: for ACAS Xu-sized DNNs, only two-dimensional input regions are supported.

In conclusion, existing repair procedures mainly consider robustness, without giving any safety guarantees. Only four methods are concerned with safety specifications: however, these either have scalability issues \cite{DBLP:conf/pldi/SotoudehT21,DBLP:conf/lpar/GoldbergerKAK20,DBLP:conf/qrs/DongSWWD21}, or do not give any formal guarantee for the resulting network \cite{DBLP:conf/icml/FischerBDGZV19}.

\paragraph{Contributions.}

To address the lack of scalable and performance-preserving neural network repair methods, we propose \emph{\tool}, an efficient and effective technique for specification-based counter-example guided repair of DNNs.

\textbf{First}, we define the \emph{satisfaction function}, an objective function that combines the function represented by the original DNN with the formal safety specification. This facilitates the search for counter-examples, i.e. inputs that lead to unsafe behavior due to violating the specification.
\textbf{Second}, we propose an approach to find these counter-examples. For that, we turn the counter-example generation problem into an optimization \cite{kochenderfer2019algorithms} problem. 
A global optimizer then carries out the specification-based counter-example search.
\textbf{Third}, we introduce an automated repair mechanism that uses the original DNN's loss function and the counter-examples from the second step to create a penalized training loss function. Additional training iterations are performed on the DNN and eliminate the counter-examples in the process while preserving high accuracy.
A verifier then checks specification compliance of the repaired network.
Crucially, the verifier is typically used only once.
\textbf{Finally}, we demonstrate the performance of \tool compared to several state-of-the-art approaches. 
The experimental results show that \tool efficiently finds counter-examples in the DNNs and successfully repairs more DNNs while also achieving better classification accuracy for the repaired DNNs.

\section{Background}\label{sec:background}

In this work, we study deep neural networks (DNNs). While our approach is independent of the particular application, 
to simplify the presentation, we restrict our attention to classification tasks.
A \emph{deep neural network} $N: \mathbb{R}^n \to \mathbb{R}^m$ assigns a given input $\vec{x} \in \mathbb{R}^n$ to confidence values $\vec{y} \in \mathbb{R}^m$ for $m$ class labels.
A DNN comprises $k$ layers that are sequentially composed such that $N = f_k \circ \dots \circ f_1$.
Each layer $i$ is assigned an activation function $\sigma_i$ and learnable parameters $\vec{\theta}$, consisting of a weight matrix $\mat{W}$ and a bias vector $\vec{b}$, such that the output of the $i$th layer is a function $f_i: \mathbb{R}^{k_{i-1}} \to \mathbb{R}^{k_i}$ with $f_i(\vec{z}) = \sigma_i(W_i \, \vec{z} + \vec{b}_i)$.

\medskip

We consider \emph{formal specifications} $\Phi = \{\varphi_1, \dots, \varphi_s\}$  composed of $s$ input-output properties.
Such a property $\varphi = (X_\varphi, Y_\varphi)$ specifies that for all points in an input set $X_\varphi$, the network needs to predict outputs that lie in an output set $Y_\varphi$.
For simplicity, we consider interval input sets $X_\varphi$ and assume that $Y_\varphi$ is a Boolean combination of constraints given in conjunctive normal form (CNF).
\begin{align}
    X_\varphi &= \left\{ \left. \prod_{i=1}^n [l_i, u_i] \right|   l_i, u_i \in \mathbb{R}, l_i \leq u_i \right\}, \label{eq:input_constraints} \\
    Y_\varphi &= \left\{ \vec{y} \in \mathbb{R}^m \left|\,\, \vec{y} \models \bigwedge_{j_1=1}^{a_\varphi}\bigvee_{j_2=1}^{b_\varphi} B_{j_1,j_2} \right. \right\}, \label{eq:output_constraints}
\end{align}
where $a_\varphi,b_\varphi \in \mathbb{N}$ are the total number of logical conjunctions and disjunctions, respectively.
The atomic constraints $B_{j_1,j_2}$ are of the form
%
\begin{align}\label{eq:atoms}
    B_{j_1,j_2} \equiv g_{j_1,j_2}(\vec{y}) \geq 0
\end{align}
where the $g_{j_1,j_2}: \mathbb{R}^m \to \mathbb{R}$ are computable functions of the output values $\vec{y}$.
Common examples of output constraints include linear constraints (such as comparing two outputs $y_1 \leq y_2$; see \Cref{tab:example-specification} for constraints used in the running example).

An $\ell_\infty$ \emph{robustness property} $\varphi_\epsilon$ is a special case of the above class of specifications.
Such a property specifies stable classification for all inputs from a hypercubic neighborhood around a given input $\vec{x}$ with radius $\epsilon$ and is defined for the desired class $c$ (i.e., the value of the corresponding output neuron is $y_c$):
\begin{align}
    X_{\varphi_\epsilon} &= \left\{ \prod_{i=1}^n [x_i - \epsilon, x_i + \epsilon] \right\}, &
    Y_{\varphi_\epsilon} &= \left\{  \vec{y} \in \mathbb{R}^m \left|\, y_c = \max_{j} y_j  \right. \right\}.
\end{align}

A DNN $N$ satisfies a property $\varphi$, resp.\ a specification $\Phi$, if the following holds:
\begin{equation}\label{eq:sat_property}
\begin{split}
    N \models \varphi &\iff \forall \vec{x} \in X_\varphi: N(\vec{x}) \in Y_\varphi \\
    N \models \Phi &\iff \forall \varphi \in \Phi: N \models \varphi.
\end{split}
\end{equation}


\subsection{Running example}

\begin{figure}[tb]
    \centering
    \begin{subfigure}[t]{0.45\textwidth}
        \hspace{-0.45cm}
        \includegraphics[width=0.9\textwidth]{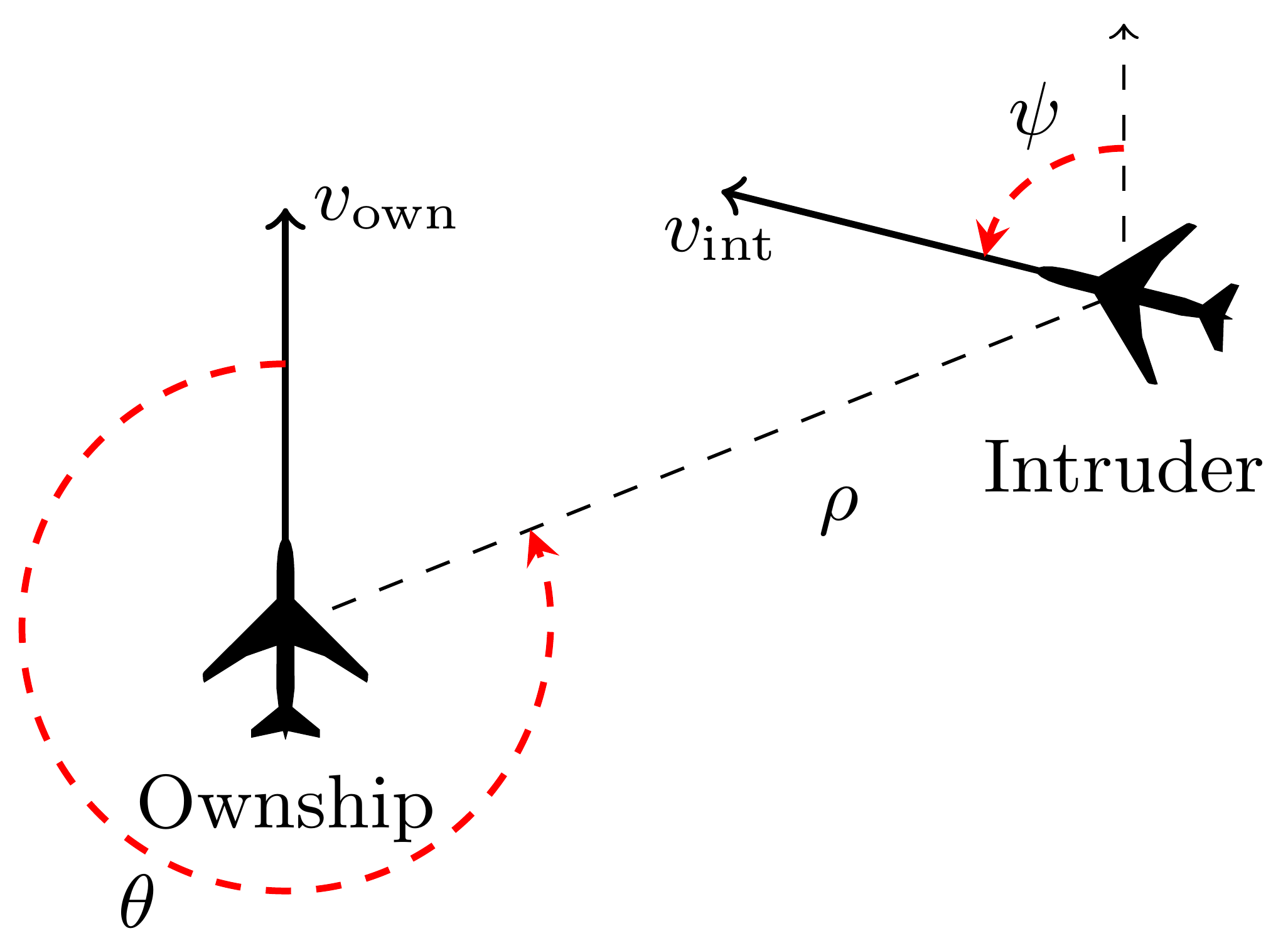}
        \label{subfig:acasxu-inputs}
    \end{subfigure}
    \begin{subfigure}[t]{0.49\textwidth}
        \includegraphics[width=1.15\textwidth]{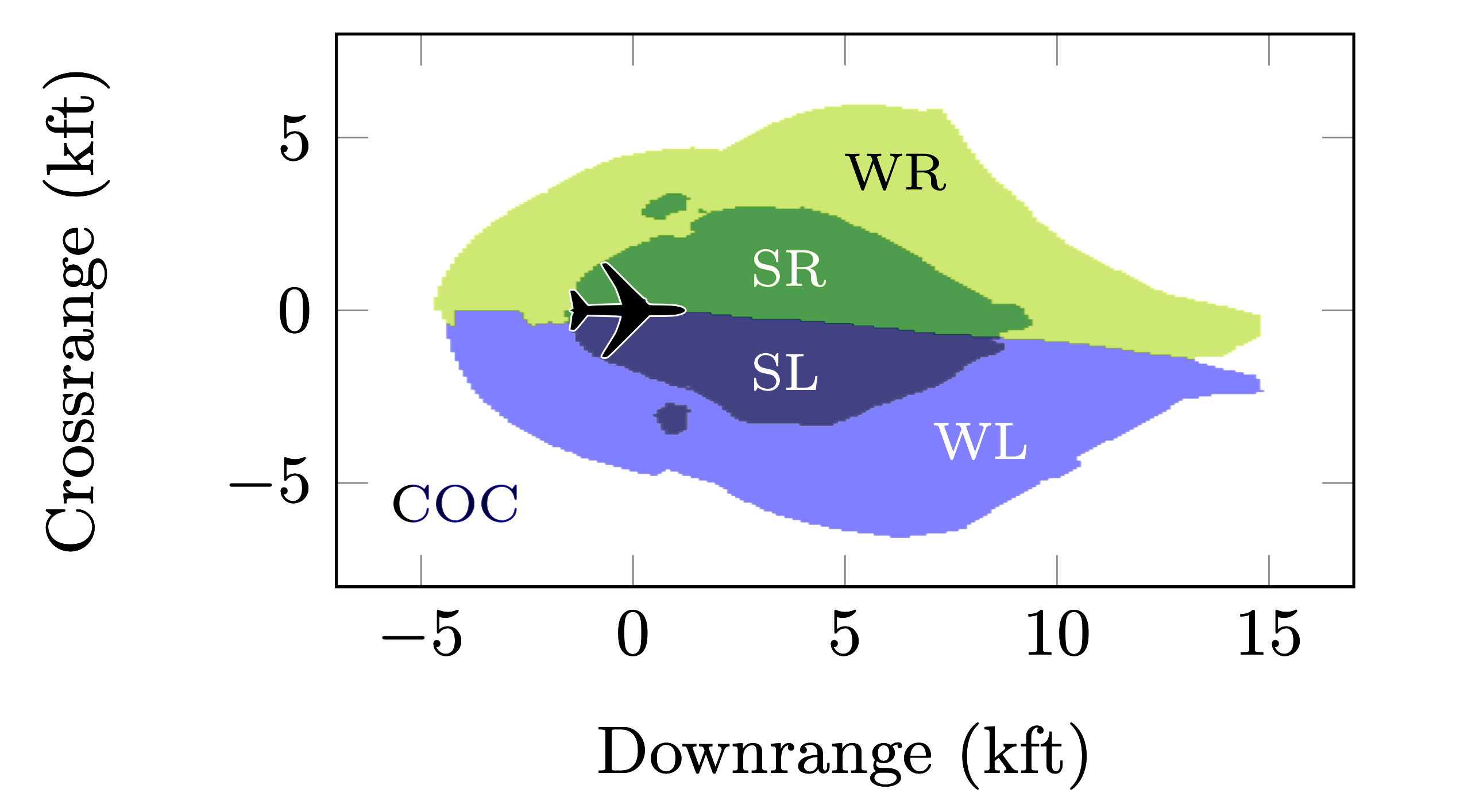}
        \label{subfig:acasxu-policies}
    \end{subfigure}
    \caption{\textbf{Running example description.} \textit{Left}: Input variables for the ACAS Xu DNNs \cite{julian2016policy}: Distance from ownship to intruder $\rho$, angle to intruder relative to ownship heading direction $\theta$, heading angle of intruder relative to ownship heading direction $\psi$, speed of ownship $v_\text{own}$, and speed of intruder $v_\text{int}$.
    \textit{Right}: Output advisories: Clear-of-conflict (COC), weak left (WL), weak right (WR), strong left (SL), and strong right (SR). Both aircraft are in the same horizontal plane. Crossrange is perpendicular to the flight direction, while downrange is horizontal to the flight direction.}
    \label{fig:acasxu-overview-vis}
\end{figure}

\emph{ACAS Xu} \cite{julian2016policy} is a system for collision avoidance of two aircraft, consisting of 45 fully connected DNNs.
Five inputs describe the relative position and speed of the two aircraft, while the outputs are five advisories, shown in \Cref{fig:acasxu-overview-vis}.
Two additional parameters, the time until loss of vertical separation $\tau$ and the previous advisory $a_\text{prev}$, are used to index which of the 45 DNNs applies to the specific scenario.
The advisory that is suggested corresponds to the DNN output with the minimum value.

\begin{table}[tb]
    \centering
    \begin{tabularx}{\textwidth}{@{}lXl@{}} \toprule
        Spec \hspace{1cm} & $X_{\varphi}$ & $Y_{\varphi}$ \\ \midrule
        $\varphi_1$
        & $[55947.691, \infty] \times \mathbb{R}^2 \times [1145, \infty] \times [-\infty, 60]$ 
        & $\left\{ \vec{y} \left|\, y_1 \leq 1500 \right. \right\}$ \\
        $\varphi_2$ 
        & $[55947.691, \infty] \times \mathbb{R}^2 \times [1145, \infty] \times [-\infty, 60]$ 
        & $\left\{ \vec{y} \left|\, y_1 \leq \max_{i \neq 1} y_i \right. \right\}$ \\
        \bottomrule
    \end{tabularx}
    \caption{\textbf{Running example properties.} ACAS Xu safety properties; If the intruder is distant and is significantly slower than the ownship, $\varphi_1$ ``Clear-of-conflict (COC, $y_1$) is always below 1500'' and $\varphi_2$ ``Clear-of-conflict is never the maximum output'', taken from \cite{DBLP:conf/cav/KatzBDJK17}.}
    \label{tab:example-specification}
\end{table}

To specify the safe behavior of the system, 10 safety properties have been formulated  \cite{DBLP:conf/cav/KatzBDJK17} (see \iftoggle{arxiv}{\Cref{sec:acas-xu-properties}}{Sect. A of the supplementary material \cite{supplementary}}).
Two example properties are given in \Cref{tab:example-specification}.

\definecolor{ra_0}{rgb}{0.258, 0.258, 0.515}
\definecolor{ra_1}{rgb}{0.5, 0.5, 1.0}
\definecolor{ra_2}{rgb}{0.894, 0.102, 0.102}
\definecolor{ra_3}{rgb}{0.806, 0.913, 0.456}
\definecolor{ra_4}{rgb}{0.306, 0.612, 0.306}

\begin{figure}[tb]
    \centering
    \begin{subfigure}{\textwidth}
        \centering
        \begin{tikzpicture}[]
        \begin{groupplot}[
            height={4.4cm}, width={6cm}, 
            group style={horizontal sep=1.5cm, 
            group size=2 by 1}
        ]
        \nextgroupplot [
            ylabel = {Crossrange (kft)}, 
            title = {Before Repair}, 
            xlabel = {Downrange (kft)}, 
            enlargelimits = false, axis on top,
        ]
        \addplot [
            point meta min=-3, point meta max=3
        ] graphics [
            xmin=-17.0, xmax=87.0, ymin=-52.0, ymax=52.0
        ] {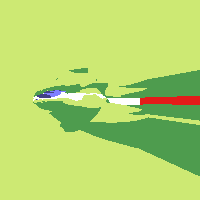};
        \node[aircraft top,
            fill=black, draw=white, 
            minimum width=2.0cm, rotate=0.0, scale = 0.35
        ] at (axis cs:0.0, 0.0) {};;
        \node[aircraft top,
            fill=red, draw=black, 
            minimum width=2.0cm, rotate=-70, scale = 0.35
        ] at (axis cs:77.64, 42.64) {};;
        
        \nextgroupplot [
             title = {After Repair}, 
             xlabel = {Downrange (kft)}, 
             enlargelimits = false, axis on top
        ]
        \addplot [
            point meta min=-3, point meta max=3
        ] graphics [
            xmin=-17.0, xmax=87.0, ymin=-52.0, ymax=52.0
        ] {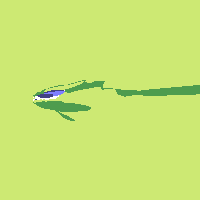};
        \node[aircraft top,
            fill=black, draw=white, 
            minimum width=2.0cm, rotate=0.0, scale = 0.35
        ] at (axis cs:0.0, 0.0) {};;
        \node[aircraft top,
            fill=red, draw=black, 
            minimum width=2.0cm, rotate=-70, scale = 0.35
        ] at (axis cs:77.64, 42.64) {};;
        \end{groupplot}
        \end{tikzpicture}
    \end{subfigure}
    \\[.25cm]
    \begin{subfigure}{\textwidth}
        \scriptsize
        \centering
        \begin{tikzpicture}
        \begin{groupplot}[height={2cm}, width={12cm}, group style={horizontal sep=2cm, group size=1 by 1}]
        \nextgroupplot [hide axis = {true}, xshift=-1.4cm, enlargelimits = false, axis on top]
        \addplot [point meta min=-2.0, point meta max=2.0] 
        graphics [xmin=-2, xmax=2, ymin=-2, ymax=2] {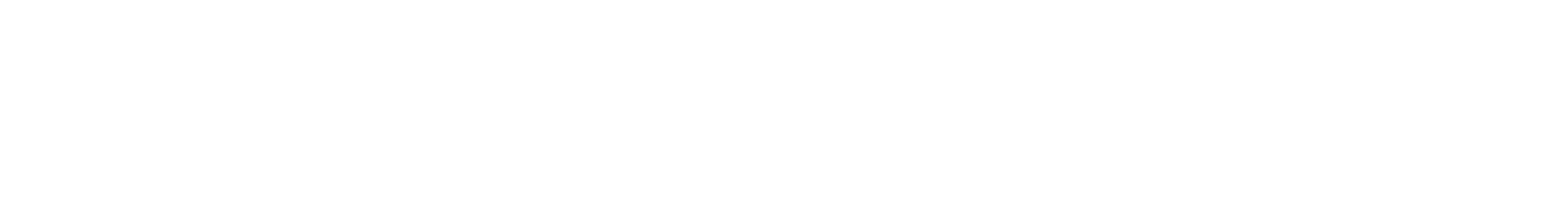};
        \addplot+[scatter, 
             point meta min=0.0, point meta max=30.0,
             scatter src=explicit symbolic, 
             only marks = {true}, 
             scatter/classes = {{
                 ra_0={style={black, fill=ra_0, mark size=3, line width=0.25pt}},
                 ra_1={style={black, fill=ra_1, mark size=3, line width=0.25pt}},
                 ra_2={style={black, fill=ra_2, mark size=3, line width=0.25pt}},
                 ok={style={black, fill=white, mark size=3, line width=0.25pt}},
                 ra_3={style={black, fill=ra_3, mark size=3, line width=0.25pt}},
                 ra_4={style={black, fill=ra_4, mark size=3, line width=0.25pt}},
                 ra_5={style={white, mark size=3}}
              }}
        ] coordinates {
        ( 6.0, 0.0) [ra_0]
        (11.0, 0.0) [ra_1]
        (16.0, 0.0) [ra_2]
        (26.5, 0.0) [ok]
        (36.0, 0.0) [ra_3]
        (41.0, 0.0) [ra_4]
        (46.0, 0.0) [ra_5]
        };
        \node at (axis cs: 6.5, 0.0) [black,anchor=west] {SL};
        \node at (axis cs:11.5, 0.0) [black,anchor=west] {WL };
        \node at (axis cs:16.5, 0.0) [black,anchor=west] {COC (unsafe)};
        \node at (axis cs:27, 0.0) [black,anchor=west] {COC (safe)};
        \node at (axis cs:36.5, 0.0) [black,anchor=west] {WR };
        \node at (axis cs:41.5, 0.0) [black,anchor=west] {SR };
        \end{groupplot}
        \end{tikzpicture}
    \end{subfigure}
    \caption{\textbf{\emph{Least} advised actions of ACAS Xu network \(N_{2,1}\) before (left) and after repair (right).}
    The advised actions are described in \Cref{fig:acasxu-overview-vis}.
    Both aircraft are in the same horizontal plane. Crossrange is perpendicular to the flight direction, while downrange is horizontal to the flight direction.
    This figure visualizes property $\varphi_2$ before and after repair.
    Here, \({\tau = 0}\), \({a_\text{prev} = \text{weak left}}\),
    \({\psi = -70^{\circ}}\), \({v_\mathrm{own} = 1185.0}\), and \({v_\mathrm{int} = 7.5}\).
    The \textbf{red area} 
    shows unsafe behavior according to $\varphi_2$ and thus constitutes counter-examples.}
    \label{fig:repair-before-after-vis}
\end{figure}

\Cref{fig:repair-before-after-vis} illustrates the goal of our paper using the ACAS Xu example: eliminate counter-examples from a given DNN by performing an automated specification-based repair.
In the figure we see that the original network (left) gives an unsafe advisory in the red region, while the repaired network (right) only gives safe advisories.

\section{\tool Overview}\label{sec:overview}

\begin{figure}[tb]
    \centering
    \includegraphics[width=\textwidth,keepaspectratio]{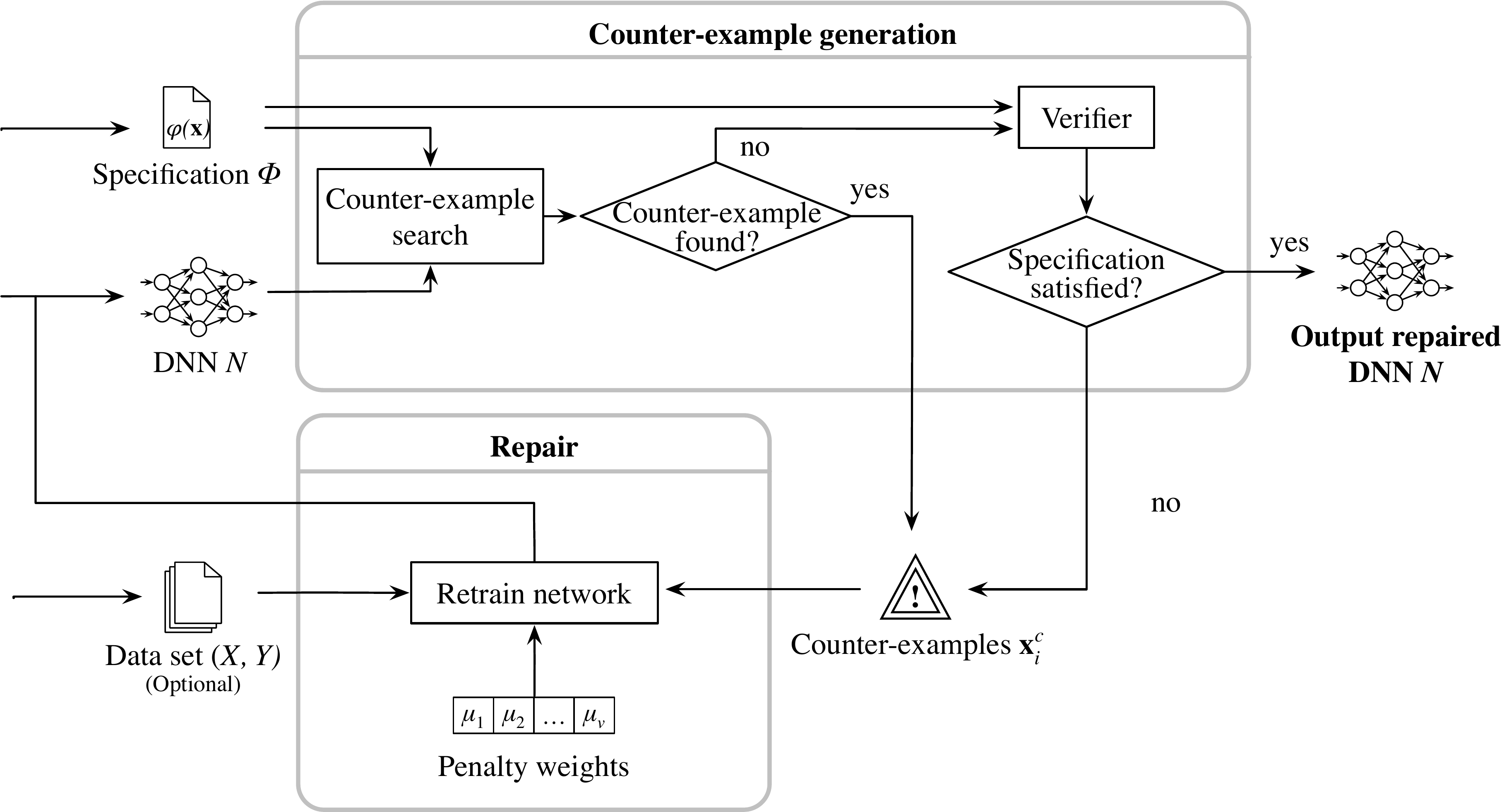}
    \caption{\tool architecture.}
    \label{fig:architecture}
\end{figure}

In this section, we give a high-level overview of our approach called \emph{\tool}.
A detailed explanation follows in the later sections.
The general structure of \tool is depicted in \Cref{fig:architecture}.
\tool iterates back and forth between the two main components \textit{counter-example generation} and \textit{repair} until it terminates after a fixed number of repair steps.

The counter-example generation component takes a formal specification $\Phi = \{\varphi_1, \dots, \varphi_s\}$ and a DNN $N$ and produces $s$ counter-examples $\vec{x}_1^c, \dots, \vec{x}_s^c$.
Subsequently, the repair component retrains the DNN, for which it uses the original data set $(X, Y)$ that $N$ was trained with, the counter-examples obtained in the last step, and a penalty weight $\mu_i$, which steers counter-example removal, starting with the original parameters $\vec{\theta}$. 
This strategy balances counter-example removal from the DNN and classification accuracy.
After re-training has taken place, the counter-example generation component is executed again. 
If no counter-example is detected, as a final step, we attempt to verify the DNN using a formal verification method.
Since the previous counter-example search is fast but incomplete (i.e., may miss counter-examples), the verifier may still find a counter-example, in which case \tool goes back to the repair component.
Otherwise, the repaired DNN $N$ is verified and returned by \tool.

In \Cref{sec:counter-example-generation} we describe how counter-example generation for a DNN $N$ with respect to a formal specification $\Phi$ is performed.
In \Cref{sec:repair-framework} we explain the counter-example guided repair approach.

\section{Finding Violations of Safety Specifications}\label{sec:counter-example-generation}

In this section, we show how the existence of a counter-example can be cast as an optimization problem.
This allows us to use an optimization procedure to find counter-examples.

\subsection{An Optimization View on Safety Specifications}

Here we show how to map a specification to an objective function, which we call the \emph{satisfaction function} $\fsat$.

\begin{definition}
The satisfaction function for an atomic constraint $B_{j_1,j_2}$ of the form $g_{j_1,j_2}(\vec{y}) \geq 0$ from~\eqref{eq:atoms} is defined as
\begin{equation}\label{eq:fsat_atomic}
    \fsat_{B_{j_1,j_2}}(\vec{y}) = g_{j_1,j_2}(\vec{y}).
\end{equation}

The satisfaction function for a given specification, i.e., set of input-output properties $\varphi = (X_\varphi, Y_\varphi)$ of the form~\eqref{eq:input_constraints} and~\eqref{eq:output_constraints}, is defined as
\begin{equation}\label{eq:fsat}
    \fsat(\vec{x}) := \underset{\varphi \in \Phi}{\min}~\underset{j_1 \in \{1 \dots a_\varphi\}}{\min}~ \underset{j_2 \in \{1 \dots b_\varphi\}}{\max}~ \fsat_{B_{j_1,j_2}}(N(\vec{x})).
\end{equation}
\end{definition}

In the following we focus on a single property $\varphi$.
Given a property $\varphi$, the satisfaction function is negative if and only if $\varphi$ is violated, which is summarized in the following theorem.

\begin{theorem}\label{th:specification-violation}
    Given a satisfaction function $\fsat$ obtained from a network $N$ and
    an input-output property $\varphi = (X_\varphi, Y_\varphi)$, we have
    \begin{equation*}
        N \not\models \varphi \iff \exists \vec{x} \in X_\varphi: \fsat(\vec{x}) < 0.
    \end{equation*}
\end{theorem}

\begin{proof}
    Fix a network $N$ and a property $\varphi = (X_\varphi, Y_\varphi)$.
    Clearly, we have
    \begin{equation}
        \text{$\fsat_{B_{j_1,j_2}}(\vec{y})$ is negative if and only if $B_{j_1,j_2}$ is violated for $\vec{y}$}. \tag{$*$}
    \end{equation}

    First suppose that $N \not\models \varphi$.
    According to~\eqref{eq:sat_property}, there exists an input $\vec{x} \in X_\varphi$ such that $N(\vec{x}) \notin Y_\varphi$.
    Since the output constraints $Y_\varphi$ in~\eqref{eq:output_constraints} are given in conjunctive normal form, one of the disjunctions and hence all corresponding disjuncts must be violated.
    By ($*$) we have that $\fsat_B$ is negative for all these disjuncts.
    Thus the $\max$ in~\eqref{eq:fsat} and hence the image of $\fsat$ itself is negative too.
    
    Now suppose that $N \models \varphi$.
    Then for each input $\vec{x} \in X_\varphi$ we have that $N(\vec{x}) \in Y_\varphi$.
    By a similar argument as above, in each disjunction there is at least one disjunct that is satisfied.
    Using ($*$), we know that $\fsat_B$ is non-negative for this disjunct.
    Finally, from~\eqref{eq:fsat} we get that $\fsat(\vec{x})$ is non-negative.
    \qed
\end{proof}

\subsection{Using Optimization To Find Counter-examples}

Using \Cref{th:specification-violation}, for detecting \emph{counter-examples} we can now equivalently minimize the function $\fsat$ in search of values below zero.
The examples in \Cref{fig:sat-function} show this for the topology of the DNN outputs, compared to the satisfaction function for the properties $\varphi_1$ and $\varphi_2$ from the running example in \Cref{tab:example-specification}.

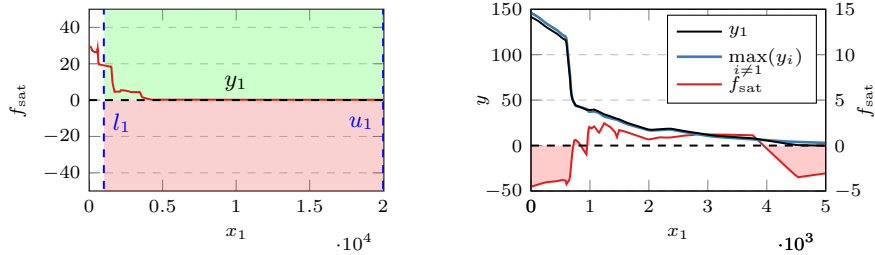
\begin{figure}[tb]
    \centering

    \begin{subfigure}[b]{0.46\textwidth}
        \begin{tikzpicture}
            \begin{axis}[
                width=5.5cm,
                height=4cm,
                xlabel={$x_1$},
                ylabel={$\fsat$},
                xmin=0, xmax=20050,
                ymin=-50, ymax=50,
                xtick={0, 5000, 10000, 15000, 20000},
                ytick={-40, -20, 0, 20, 40},
                legend style={font=\scriptsize,anchor=north west},
                ymajorgrids=true,
                grid style=dashed,
                label style={font=\scriptsize},
                tick label style={font=\scriptsize}  
            ]
            \fill [Colors-A,opacity=0.2] (axis cs:1000,-50) rectangle (axis cs:60700,0);
    
            \fill [green,opacity=0.2] (axis cs:1000,0) rectangle (axis cs:60700,50);
            
            \addplot[Colors-A,line width=0.75pt] table
            [x=x,
            y expr=(\thisrowno{1}-1500), 
            col sep=comma] {data/sample_file.csv};
            
            \draw [dashed,blue,line width=0.75pt] (axis cs:1000,-75) -- node[right]{$l_1$} (axis cs:1000,50);
            \draw [dashed,blue,line width=0.75pt] (axis cs:20000,-75) -- node[left]{$u_1$} (axis cs:20000,50);
            
            \draw [dashed,black,line width=0.75pt] (axis cs:0,0) -- node[above]{$y_1$} (axis cs:20000,0);
            \draw [dashed,black,line width=0.75pt] (axis cs:0,180) -- node[above]{$y_u$} (axis cs:10000,180);
            \end{axis}
        \end{tikzpicture}
    
        \caption{\textbf{Property $\mathbf{\varphi_1}$.} 
        $\fsat = 1500 - y_1$ encodes the specification $\varphi_1$: ``The score for $y_1$ is always below 1500''.
        The specification also includes input constraint $l_1 \leq x_1 \leq u_1$.
        }
        \label{fig:sat-function-box}
    \end{subfigure}\hspace{0.5cm}
    \begin{subfigure}[b]{0.46\textwidth}
        \begin{tikzpicture}
            \begin{axis}[
                width=5.5cm,
                height=4cm,
                xlabel={$x_1$},
                ylabel={$y$},
                ylabel shift=-8pt,
                xmin=0, xmax=5000,
                ymin=-50, ymax=150,
                scaled x ticks={base 10:-3},
                xtick={0, 1000, 2000, 3000, 4000, 5000},
                ytick={-50, 0, 50, 100, 150},
                legend cell align={left},
                legend style={font=\scriptsize,anchor=north west},
                label style={font=\footnotesize},
                tick label style={font=\footnotesize},
                reverse legend,
                ymajorgrids=true,
                grid style=dashed,
                legend pos=north east,
                label style={font=\scriptsize},
                tick label style={font=\scriptsize}  
            ]
            
            \addplot[Colors-A,line width=0.75pt,name path=A] table [x=x, y=y1minusy2, col sep=comma] {data/sample_file_3d.csv};
            \addlegendentry{$\fsat$}
            
            \addplot[Colors-C,line width=1pt] table [x=x, y=y2, col sep=comma] {data/sample_file_3d.csv};
            \addlegendentry{$\max\limits_{i \neq 1}(y_i)$}
           
            \addplot[black,line width=0.75pt] table [x=x, y=y1, col sep=comma] {data/sample_file_3d.csv};
            \addlegendentry{$y_1$}
            
            \draw [name path=B,dashed,black,line width=0.75pt] (axis cs:0,0) -- node[below]{} (axis cs:5000,0);
            
            \addplot[red,opacity=0.2,forget plot] fill between[of=A and B,soft clip={domain=0:1000}]; 
            
            \addplot[red,opacity=0.2,forget plot] fill between[of=A and B,soft clip={domain=4000:5000}]; 
    
            \end{axis}
            \begin{axis}[
                width=5.5cm,
                height=4cm,
                xmin=0, xmax=5000,
                ymin=-5, ymax=15,
                scaled x ticks={base 10:-3},
                xtick={0},
                ylabel={$\fsat$},
                ylabel shift=-8pt,
                yticklabel pos=right,
                ytick={-5, 0, 5, 10, 15},
                ytick style={draw=none},
                label style={font=\footnotesize},
                tick label style={font=\footnotesize},
                ymajorgrids=false,
                tickwidth=0,
                grid style=none,
                label style={font=\scriptsize},
                tick label style={font=\scriptsize}  
            ]
            \end{axis}
        \end{tikzpicture}
        \caption{\textbf{Property $\mathbf{\varphi_2}$.} $\fsat = -y_1 + \max_{i \neq 1}(y_i)$ encodes specification $\varphi_2$: ``$y_1$ is never the maximum value''.
        The red curve has been scaled by a factor of 10 for better visibility.}
        \label{fig:sat-function-min-max}
    \end{subfigure}
    \caption{\textbf{Running example.} Safety properties on the ACAS Xu DNNs are mapped to satisfaction functions $\fsat$, which map counter-examples (safety violations) to negative values, shaded in red.
    The two examples illustrate the properties $\varphi_1$ and $\varphi_2$ from \Cref{tab:example-specification}.}
    \label{fig:sat-function}
\end{figure}

The satisfaction function $\fsat$ enables us to turn the problem of finding counter-examples witnessing a specification violation of a DNN into a multivariate optimization problem.
Note that both the DNN and the specification are fully captured by $\fsat$ and hence we can call any off-the-shelf black-box optimization algorithm with $X_\varphi$ as the input bounds and $\fsat$ as the objective function to be minimized.
The optimization procedure used here \cite{Endres18} was chosen experimentally. For details refer \iftoggle{arxiv}{\Cref{sec:global-optimization-method}}{to Sect. B of the supplementary material \cite{supplementary}}.
Optimization tools are efficient in driving a function, here $\fsat$, toward its minimum; hence our approach often finds counter-examples much faster than other approaches.

\section{Repair Framework}\label{sec:repair-framework}

In the previous section, we have seen how to find counter-examples that violate the specification of a DNN.
In this section, we build a framework around that algorithm to \emph{repair} the DNN.
By ``repair'' we mean to modify the network parameters such that the new DNN satisfies the specification.
However, modifying the network parameters generally changes the accuracy of the DNN as well.
Thus, as a second goal, we intend to preserve the accuracy of the DNN as much as possible.
Our repair technique can be summarized as follows: 
\textit{minimize} the loss in accuracy of the DNN \textit{such that} the DNN satisfies the given specification.

\medskip

DNN training uses unconstrained optimization of a \emph{loss function}~\cite{DBLP:books/daglib/0040158}.
In contrast, \emph{constrained} optimization problems can be stated as follows:
\begin{align*}
    \text{minimize } f(\vec{\theta}) \text{ such that } c_i(\vec{\theta}) \geq 0 \text{ for all } i \in \{1, \dots, v\}.
\end{align*}
Here, $f$ is a loss function and the $c_i$ are constraints under which a point $\vec{\theta}$ is admissible to the problem.

We introduce constraints into the training procedure by incorporating penalty functions \cite{smith1997penalty} into the loss function.
We want to minimize this loss function such that it satisfies the additional constraints $c_i$, which are assigned a positive penalty weight $\mu_i$, defining the penalized objective function problem as
\begin{equation}\label{eq:penalized-objective-function}
    \underset{\vec{\theta}}{\argmin} \,\, f(\vec{\theta}) + \sum_{i=1}^v \mu_i \cdot c^+_i(\vec{\theta})
\end{equation}
where penalty function $c^+_i$ is defined as $c^+_i(\vec{\theta}) = \max(0, -c_i(\vec{\theta}))$.
Intuitively, the penalty function forces the unconstrained optimization algorithm, which is used to solve \Cref{eq:penalized-objective-function}, to minimize the constraint violation: If a constraint is violated, it adds a large positive term to the objective function.

By enhancing the training procedure using the penalized loss function defined in  \Cref{eq:penalized-objective-function}, training both incorporates model accuracy (since the old loss function is part of the new loss function) and decreases the violation of the counter-examples $\vec{x}_1^c, \dots, \vec{x}_v^c$.
After each training iteration, the penalty weights $\vec{\mu}$ are updated, and the current model parameters are used as starting points for subsequent iterations.

\medskip

Algorithm~\ref{algo:penalty-function} gives a detailed view of one repair step.
As inputs, the algorithm takes a DNN $N$ to repair with weights $\vec{\theta}$, the original data set to train $N$ (if not available, a uniform sampling of $N$ can be used), a set of counter-examples $\vec{x}^c_i$ for the set of safety properties $\varphi_i$ that have been found during the counter-example generation step, and the training loss function $\lambda$ originally used to train $N$.

\begin{algorithm}[t]
 \KwIn{DNN \(N: \mathbb{R}^n \to \mathbb{R}^m\) with parameters \(\vec{\theta}\), data set $(X, Y)$, counter-examples and properties \((\mathbf{x}^c_{1}, \varphi_{1}), (\mathbf{x}^c_{2}, \varphi_{2}), \ldots, (\mathbf{x}^c_{v}, \varphi_{v})\), loss function \(\lambda: \mathbb{R}^{\text{dim}(\vec{\theta})} \to \mathbb{R}\).}
 \KwData{Penalty weights \(\mu_{1}, \mu_{2}, \ldots, \mu_{v}\), constraint functions \(c_i: \mathbb{R}^{\text{dim}(\vec{\theta})} \to \mathbb{R}\) for \(i \in \{1, \ldots, v\}\), penalized loss function \(\lambda': \mathbb{R}^{\text{dim}(\vec{\theta})} \to \mathbb{R}\).}
 \KwOut{Repaired DNN \(N\) with new parameters \(\vec{\theta'}\).}
 \ForEach{\(i \in \{1, \ldots, v\}\)}{
    \(\mu_i \leftarrow 1\)\tcp*{default initial penalty weight}
    
    \(c_i(\vec{\theta}') \leftarrow \fsat_{\varphi_i}(N_{\vec{\theta}'}(\mathbf{x}^c_i))\) \label{alg:line:fsat}\;
 }
 \While{\(\exists j \in \{1, \ldots, v\}: N(\mathbf{x}^c_{j}) \notin Y_{\varphi_{j}}\)}{
 	\(\lambda'(\vec{\theta}') \leftarrow \lambda(\vec{\theta}') + \sum_{i=1}^v \mu_i \cdot c_i^+(\vec{\theta}')\) \label{alg:line:loss} \;
 	
    \text{train}(\(N, (X, Y), \lambda'\))\;
     
    \ForEach{\(i \in \{1, \ldots, v\}\)}{
        \If{\(N(\mathbf{x}^c_i) \notin Y_{\varphi_i}\)}{
            \(\mu_i \leftarrow 2\mu_i\)\tcp*{default penalty increase strategy}
        }
    }
 }
 \caption{Penalty function repair.}
 \label{algo:penalty-function}
\end{algorithm}

The algorithm first iterates over all the counter-example/property pairs that it was given, assigning each counter-example $\vec{x}_i^c$ an initial penalty weight $\mu_i = 1$.
Then, we build the constraint function $c_i$ by using the $\fsat$ function and applying it to the counter-example and current network weights (line~\ref{alg:line:fsat}).

In the second loop, the DNN $N$ is iteratively trained.
In line~\ref{alg:line:loss}, each counter-example is converted into a penalized constraint.
Each unsatisfied constraint adds a positive term to the loss function's objective value.
Therefore, the loss function is likely not minimal when there are any unsatisfied constraints.
This updated loss function is then used to re-train the DNN.
After training, we check if the counter-examples still occur in the re-trained DNN. If so, the penalty weight is doubled; otherwise, the successfully repaired DNN with new weights $\vec{\theta'}$ is returned.

The \textit{counter-example generation} component outputs counter-examples and hands them over to the \textit{repair} component.
This process is repeated iteratively.
As outlined in \Cref{sec:counter-example-generation}, the satisfaction function $\fsat$ is input to a global optimization algorithm.
Intuitively, it would be possible to exit the optimization routine early when any negative value is detected.
However, in our experiments we found that taking the counter-examples at the minimum, expressing a higher violation severity, ultimately results in more successful repairs.

\section{Evaluation}\label{sec:evaluation}

This section presents our experimental evaluation, demonstrating the algorithm's effectiveness in repairing a neural network subject to a safety specification.
Our implementation of \tool uses PyTorch for DNN interactions.
We conduct the experiments on an Intel Xeon E5-2680 CPU with 2.4\,GhZ and 170\,GB of memory.
For the final verification step, our implementation uses the verifier ERAN \cite{DBLP:journals/pacmpl/SinghGPV19}.
We note that ERAN uses an internal timeout and may hence return \textsc{Unknown}.
This can be circumvented by increasing the timeout, but in the evaluation we use the default settings of ERAN and give up with the result \textsc{Unknown} instead.

\subsection{Experimental Setup}

Our experiments use 36 networks for the tasks of aircraft collision avoidance and image classification, where we replicate the benchmark networks from \cite{DBLP:journals/pacmpl/SinghGPV19} for the latter.
In detail:

\begin{itemize}
    \item The collision avoidance system \emph{ACAS Xu} \cite{julian2016policy} consists of 45 fully connected DNNs, $N_{1,1}$ to $N_{5,9}$.
    The inputs and outputs are described in 
    \iftoggle{arxiv}{\Cref{tab:acasxu-overview} in \Cref{sec:acas-xu-properties}}{Sect. A of the supplementary material \cite{supplementary}}.
    Each of the 45 networks has 6 hidden layers with 50 ReLU nodes.
    We use the 34 networks that were shown to violate at least one of the safety properties from \cite{DBLP:conf/cav/KatzBDJK17} to evaluate our method.
    Because the training data is not openly available, we resort to a uniform sampling of the original model and compare the repaired model to it in terms of classification accuracy (reminder: the percentage of correct classifications) and mean average error (MAE) between the classification scores of the original and repaired models.
    \item \emph{MNIST} \cite{lecun-mnisthandwrittendigit-2010} contains 70k grayscale images, showing a handwritten digit from 0 to 9, with $28 \times 28$ pixels.
    We use a fully connected DNN trained to a test accuracy of 97.8\% using DiffAI-defended training \cite{DiffAi}, with five hidden dense layers of 100 units each.
    \item \emph{CIFAR10} \cite{cifar10} contains 60k color images, showing an object from one of ten possible classes, with $32 \times 32$ pixels.
    We use the benchmark CNN from \cite{DBLP:journals/pacmpl/SinghGPV19}, which was trained to an accuracy of 58.6\%. It has two convolutional layers \cite{DBLP:journals/neco/LeCunBDHHHJ89} and a max-pooling layer, repeated once with 24 and 32 channels, respectively, followed by two dense layers with 100 units each.
\end{itemize}

\subsection{Counter-Example-Based Repair}

We compare \tool against three state-of-the-art repair techniques that we described in \Cref{par:related-work}: minimal modification (MM) by \citet{DBLP:conf/lpar/GoldbergerKAK20}, \mbox{nRepair} (NR) by \citet{DBLP:conf/qrs/DongSWWD21}, and DL2 by \citet{DBLP:conf/icml/FischerBDGZV19}.
We do not compare against \cite{DBLP:conf/pldi/SotoudehT21} because its specification encoding is only applicable to two dimensions for small-scale networks and thus neither supports the ACAS Xu \cite{DBLP:conf/cav/KatzBDJK17} nor any image classification robustness properties.
We run DL2 with 5 different values for the DL2 weight parameter: 0.01, 0.05, 0.1, 0.2, and 0.5 (for more details see \cite{DBLP:conf/icml/FischerBDGZV19}).
We analyze all final repair outcomes with ERAN to assess whether the repairs produced by the tools are genuinely safe.

\medskip

For the \emph{collision avoidance task}, we repair different problem instances: 34 DNNs subject to three different properties.
As an additional challenge, we also create a combined specification $\Phi = \{\varphi_1, \varphi_2, \varphi_3, \varphi_4, \varphi_8\}$ consisting of five properties.
We evaluate the classification accuracy (percentage of correct classifications) and mean average error (MAE) after repair to measure the level of correct functionality.
Because we have no access to the original training and test data sets, we use a uniform sampling from the original network as test data for calculating accuracy and MAE.
We set a timeout of three hours for all techniques.

\medskip

For the \emph{image classification task}, we repair a total of 100 cases for an $\ell_{\infty}$ robustness specification, with a robustness parameter $\epsilon = 0.03$, replicating the robustness experiment by \citet{DBLP:conf/qrs/DongSWWD21}.
Additionally, for MNIST, we compare batch repair of 10 and 25 counter-examples at the same time.
This, however, is only applicable to \tool and nRepair (NR), because MM and DL2 do not have this functionality.
We do not compare against minimal modification (MM) \cite{DBLP:conf/lpar/GoldbergerKAK20} because of problems that let the internally used Marabou solver \cite{DBLP:conf/cav/KatzHIJLLSTWZDK19} fail to generate any counter-examples. We have reported this error to the Marabou developers\footnote{\url{https://github.com/NeuralNetworkVerification/Marabou/issues/494}}.
Also, the authors of MM themselves already described their technique to perform sub-optimally for CNF properties because it relies on the exact encoding needed by Marabou.
Furthermore, because nRepair (NR) does not support convolutional layers, we cannot evaluate it on the network for CIFAR10.

We compare the number of successful repairs, test accuracy (to measure preservation of the model's functionality), and runtime.

\paragraph{Results.}

\begin{table}[t]
    \begin{center}
    \begin{tabularx}{\textwidth}{@{}l cccc cc c@{}}
    \toprule
    & \multicolumn{4}{c}{\textbf{Repair Outcome}} & \textbf{Accuracy[\%]} & \textbf{MAE} & \textbf{Runtime \![s]}  \\
    \textbf{Tool}   & \textsc{Success}\, & \,\textsc{Fail}\, & \,\textsc{Unknown}\, & \,\textsc{Timeout}   & median & median & median \\
    \midrule
    \tool      & 28          &  0 &  2 &  6 &  \textbf{99.5} & \textbf{0.1}  & 573.2 \\
    NR         & \textbf{29} &  0 &  1 &  6 &  87.6          & 1996.1        & \textbf{10.0} \\
    DL2        & 3           & 33 &  0 &  0 &  93.4          & 6.03          & 10840.7 \\
    MM         & 0           &  0 &  0 & 35 &  --            & --            & 10832.1 \\
    \bottomrule
    \end{tabularx}
    \end{center}
    \caption{\textbf{Safety repair results of the ACAS Xu DNNs}: \tool (this work), DL2 \cite{DBLP:conf/icml/FischerBDGZV19}, nRepair (NR) \cite{DBLP:conf/qrs/DongSWWD21}, and minimal modification (MM) \cite{DBLP:conf/lpar/GoldbergerKAK20}.
    We compare the cumulative repair outcome for all 35 instances, test accuracy and mean average error (MAE) after repair, and the median runtime.
    `\exsuccess' indicates successful repairs, `\exfailure' indicates failed repairs, `\extimeout' indicates a timeout, and `\exunknown' marks cases where the verifier returned \textsc{Unknown}.}
    \label{tab:repair-acas-xu-overview}
\end{table}

\begin{table}
    \begin{center}
\begin{tabularx}{\textwidth}{@{}cc|ccc|ccc|ccc@{}}
\toprule
         	  &       		& \multicolumn{3}{c|}{\textbf{Repair Outcome}} & \multicolumn{3}{c|}{\textbf{Accuracy [\%]}} & \multicolumn{3}{c}{\textbf{MAE}} \\
Spec   		  & Model       & \,\tool{ }   & NR         & DL2\,        & \,\tool{ } 		& NR 			 & DL2\,  & \,\tool{ } 			& NR 			  & DL2 \\
\midrule  
\(\varphi_2\) & \(N_{2,1}\) & \exsuccess & \exsuccess & \exfailure & \textbf{99.1}  & 83.9           &  --  &   \textbf{0.22}   & 2242.6 		&    --  	\\
\(\varphi_2\) & \(N_{2,2}\) & \exsuccess & \exsuccess & \exsuccess & \textbf{98.7}  & 85.1           & 93.4 &   \textbf{0.23}   & 2279.3 		&   6.29 	\\
\(\varphi_2\) & \(N_{2,3}\) & \exsuccess & \exsuccess & \exfailure & \textbf{99.3}  & 83.5           &  --  &   \textbf{0.13}   & 2420.6 		&    --  	\\
\(\varphi_2\) & \(N_{2,4}\) & \exsuccess & \extimeout & \exfailure & \textbf{99.5}  &  --            &  --  &   \textbf{0.09}   &    --  		&    --  	\\
\(\varphi_2\) & \(N_{2,5}\) & \extimeout & \exsuccess & \exfailure &  --            & \textbf{84.1}  &  --  &    --             & \textbf{2433.8} &    --  \\
\(\varphi_2\) & \(N_{2,6}\) & \extimeout & \exsuccess & \exfailure &  --            & \textbf{85.6}  &  --  &    --             & \textbf{2303.7} &    --  \\
\(\varphi_2\) & \(N_{2,7}\) & \exsuccess & \exsuccess & \exsuccess & 14.5           & \textbf{87.0}  & 89.5 &   \textbf{0.15}   & 1644.8 		&   5.97 	\\
\(\varphi_2\) & \(N_{2,8}\) & \exsuccess & \exsuccess & \exfailure & \textbf{99.6}  & 87.3           &  --  &   \textbf{0.14}   & 663.6  		&    --  	\\
\(\varphi_2\) & \(N_{2,9}\) & \exsuccess & \exsuccess & \exfailure & \textbf{99.8}  & 88.6           &  --  &   \textbf{0.13}   & 2405.1 		&    --  	\\
\(\varphi_2\) & \(N_{3,1}\) & \exsuccess & \exsuccess & \exfailure & \textbf{98.6}  & 77.5           &  --  &   \textbf{0.27}   & 6.1    		&    --  	\\
\(\varphi_2\) & \(N_{3,2}\) & \exsuccess & \extimeout & \exfailure & \textbf{99.9}  &  --            &  --  &   \textbf{0.10}   &    --  		&    --  	\\
\(\varphi_2\) & \(N_{3,4}\) & \exsuccess & \extimeout & \exfailure & \textbf{99.5}  &  --            &  --  &   \textbf{0.10}   &    --  		&    --  	\\
\(\varphi_2\) & \(N_{3,5}\) & \exsuccess & \exsuccess & \exfailure & \textbf{99.5}  & 84.2           &  --  &   \textbf{0.09}   & 2384.2 		&    --  	\\
\(\varphi_2\) & \(N_{3,6}\) & \exunknown & \exsuccess & \exfailure &  --            & \textbf{81.8}  &  --  &    --             & \textbf{2387.0} &    --  \\
\(\varphi_2\) & \(N_{3,7}\) & \exsuccess & \exsuccess & \exfailure & \textbf{99.7}  & 87.0           &  --  &   \textbf{0.11}   & 2251.7 		&    --  	\\
\(\varphi_2\) & \(N_{3,8}\) & \exsuccess & \exsuccess & \exfailure & \textbf{99.7}  & 87.9           &  --  &   \textbf{0.09}   & 1311.1 		&    --  	\\
\(\varphi_2\) & \(N_{3,9}\) & \extimeout & \exsuccess & \exfailure &  --            & \textbf{87.2}  &  --  &    --             & \textbf{2442.5} &    --  \\
\(\varphi_2\) & \(N_{4,1}\) & \exsuccess & \exsuccess & \exfailure & \textbf{99.8}  & 87.7           &  --  &   \textbf{0.11}   & 1939.3 		&    --  	\\
\(\varphi_2\) & \(N_{4,3}\) & \exsuccess & \exsuccess & \exsuccess & \textbf{99.4}  & 87.8           & 96.0 &   \textbf{0.13}   & 2419.4 		&   6.03 	\\
\(\varphi_2\) & \(N_{4,4}\) & \exsuccess & \exsuccess & \exfailure & \textbf{99.5}  & 87.9           &  --  &   \textbf{0.10}   & 1090.5 		&    --  	\\
\(\varphi_2\) & \(N_{4,5}\) & \exsuccess & \exsuccess & \exfailure & \textbf{99.4}  & 87.5           &  --  &   \textbf{0.08}   &    2.7 		&    --  	\\
\(\varphi_2\) & \(N_{4,6}\) & \exsuccess & \exsuccess & \exfailure & \textbf{99.6}  & 89.8           &  --  &   \textbf{0.07}   & 1329.1 		&    --  	\\
\(\varphi_2\) & \(N_{4,7}\) & \exsuccess & \exsuccess & \exfailure & \textbf{98.3}  & 88.9           &  --  &   \textbf{0.14}   & 1996.1 		&    --  	\\
\(\varphi_2\) & \(N_{4,8}\) & \exsuccess & \exsuccess & \exfailure & \textbf{99.1}  & 88.6           &  --  &   \textbf{0.16}   & 584.3  		&    --  	\\
\(\varphi_2\) & \(N_{4,9}\) & \exsuccess & \exsuccess & \exfailure & \textbf{99.5}  & 88.8           &  --  &   \textbf{0.06}   & 2292.2 		&    --  	\\
\(\varphi_2\) & \(N_{5,1}\) & \exsuccess & \exsuccess & \exfailure & \textbf{99.5}  & 87.5           &  --  &   \textbf{0.11}   & 2227.2 		&    --  	\\
\(\varphi_2\) & \(N_{5,2}\) & \exsuccess & \exsuccess & \exfailure & \textbf{99.7}  & 87.6           &  --  &   \textbf{0.10}   & 2438.8 		&    --  	\\
\(\varphi_2\) & \(N_{5,4}\) & \exsuccess & \exsuccess & \exfailure & \textbf{99.6}  & 87.8           &  --  &   \textbf{0.09}   & 405.1  		&    --  	\\
\(\varphi_2\) & \(N_{5,5}\) & \extimeout & \exsuccess & \exfailure &  --            & \textbf{87.9}  &  --  &    --             & \textbf{749.8}  &    --  \\
\(\varphi_2\) & \(N_{5,6}\) & \exsuccess & \extimeout & \exfailure & \textbf{99.5}  &  --            &  --  &   \textbf{0.12}   &     -- 		&    --  	\\
\(\varphi_2\) & \(N_{5,7}\) & \exsuccess & \exsuccess & \exfailure & \textbf{98.4}  & 88.0           &  --  &   \textbf{0.16}   & 957.7  		&    --  	\\
\(\varphi_2\) & \(N_{5,8}\) & \exsuccess & \exsuccess & \exfailure & \textbf{99.4}  & 87.7           &  --  &   \textbf{0.11}   & 382.5  		&    --  	\\
\(\varphi_2\) & \(N_{5,9}\) & \exsuccess & \exsuccess & \exfailure & \textbf{98.1}  & 87.9           &  --  &   \textbf{0.13}   & 181.2  		&    --  	\\
\(\varphi_7\) & \(N_{1,9}\) & \exunknown & \extimeout & \exfailure &  --  			&  --  			 &  --  &    --  			&     --  		&    --  	\\
\(\varphi_8\) & \(N_{2,9}\) & \extimeout & \exunknown & \exfailure &  --  			&  --  			 &  --  &    --  			&     -- 		&    --  	\\
\(\Phi_1\)    & \(N_{2,9}\) & \extimeout & \extimeout & \exfailure & --  			& --  	   		 &  --  &    --  			&    -- 		&    -- 	\\
\midrule
           	  &             &         28 & \textbf{29}&          3 &\textbf{99.5}   & 87.6  & 93.4  &    \textbf{0.1}  & 1996.1 &  6.03  \\
              &             & \multicolumn{3}{c|}{\textsc{Success} frequency} & \multicolumn{3}{c|}{median} & \multicolumn{3}{c}{median} \\
\bottomrule
\end{tabularx}
\end{center}
    \caption{\textbf{Safety repair results of the ACAS Xu DNNs}: \tool (this work), nRepair (NR) \cite{DBLP:conf/qrs/DongSWWD21}, and DL2 \cite{DBLP:conf/icml/FischerBDGZV19}.
    Each row shows the results for one benchmark instance with the property/specification in the first column (see \iftoggle{arxiv}{\Cref{sec:acas-xu-properties}}{Sect. A of the supplementary material \cite{supplementary}}) and the DNN in the second column (names taken from~\cite{julian2016policy}).
    `\exsuccess' indicates successful repairs, `\exfailure' indicates failed repairs, `\extimeout' indicates a timeout, and `\exunknown' marks cases where the verifier returned \textsc{Unknown}.}
    \label{tab:repair-acas-xu}
\end{table}

\Cref{tab:repair-acas-xu-overview} shows the aggregated repair results for the collision avoidance task, counting the number of successful repairs, failures, unknown outcomes, and timeouts.
Additionally, it shows the accuracy and mean average error (MAE) of the DNNs after they have been repaired by the respective method. Results per benchmark instance are given in \Cref{tab:repair-acas-xu}.

\tool successfully repairs 28 of the 36 instances.
The DNNs that have been repaired by \tool achieve the highest classification accuracy with 99.5 and lowest mean average error with 0.1.
Two times ERAN (and hence \tool) terminates with a result of \textsc{Unknown}.
Six times, \tool could not repair the DNN within the time limit.

nRepair (NR) repairs one more instance than \tool, but at the cost of yielding the lowest test accuracy of the three successful tools. Also, the mean average error (MAE) is extremely high: NR does not consider the classification scores but instead is only concerned with maintaining the correct class, leading to large deviations from the original policy.

While DL2's repair accuracy is still reasonably high, with a median value of 93.4 and MAE of 6.03, it only delivers three successful repairs and fails to repair the other 33 instances. This is likely due to DL2's algorithm design, which includes a hard-coded cross-entropy loss function and no termination criteria beyond performing a large number of iterations. Experiments using DL2 with a task loss function did not result in any successful repair.

Minimal modification (MM) times out for every instance.
We explain this with the high computational cost using the SMT-based method Marabou \cite{DBLP:conf/cav/KatzHIJLLSTWZDK19} to modify network parameters directly.

\begin{table}[t]
    \begin{center}
    \begin{tabularx}{\textwidth}{X cccc c c}
    \toprule
         & \multicolumn{4}{c}{\textbf{Repair Outcome}}&  \textbf{Accuracy \![\%]} & \textbf{Runtime [s]} \\
    \textbf{Tool} & \,\,\textsc{Success}\,\, & \,\,\textsc{Fail}\,\, & \,\,\textsc{Unknown}\,\, & \,\,\textsc{Timeout}            & median  & median \\
    \midrule
    \tool   & \textbf{100} &   0 &   0 &   0 & 96.0           & 163.6\\
    NR      & 84           &  16 &   0 &   0 & \textbf{97.3}  & \textbf{36.0} \\
    DL2     & 10           &  82 &   0 &   8 & 91.9           & 18839.2 \\
    \bottomrule
    \end{tabularx}
    \end{center}
    \caption{\textbf{Robustness repair results of the DiffAI-defended \cite{DiffAi} MNIST DNN}: \tool (this work), nRepair (NR) \cite{DBLP:conf/qrs/DongSWWD21}, and DL2 \cite{DBLP:conf/icml/FischerBDGZV19}.
    We compare the cumulative repair outcome for all 100 instances, the test accuracy (minimum, median, and maximum) after repair, and the median runtime.}
    \label{tab:repair-mnist}
\end{table}

\begin{table}[t]
\begin{center}
\footnotesize
\begin{tabularx}{\textwidth}{@{}X cc @{\hskip 1em} cccc c@{}}
\toprule
       &   &  &  \multicolumn{4}{c}{\textbf{Repair Outcome}} & \textbf{Accuracy \![\%]} \\
\textbf{Tool} & \textbf{Points} & \textbf{Inst.}  & \,\,\textsc{Success}\,\, & \,\,\textsc{Fail}\,\, & \,\,\textsc{Unknown}\,\, & \,\,\textsc{Timeout}    & median   \\
\midrule
        & 1 & 100 & \textbf{100}  &   0 &   0 &   0 & 96.0          \\
\tool   & 10 & 10 & \textbf{10}   &   0 &   0 &   0 & \textbf{93.1} \\
        & 25 & 4  & \textbf{3}    &   0 &   1 &   1 & \textbf{93.5} \\
\midrule
        & 1  & 100 &  84 &  16 &   0 &   0 & \textbf{97.3} \\
NR      & 10 & 10  &   0 &  10 &   0 &   0 & -- \\
        & 25 & 4   &   0 &  4  &   0 &   0 & -- \\
\bottomrule
\end{tabularx}
\end{center}
\caption{\textbf{Collective robustness repair results of the DiffAI-defended \cite{DiffAi} MNIST DNN}: \tool (this work), and nRepair (NR) \cite{DBLP:conf/qrs/DongSWWD21}.
    We compare the cumulative repair outcome for all 100 instances, supplemented by two partitions into groups of 10 and 25 points per instance, and the test accuracy (minimum, median, and maximum) after repair.}
        \label{tab:repair-mnist-instances}
\end{table}

\smallskip

For repairing the DiffAI-defended \cite{DiffAi} MNIST DNNs, the results are presented in \Cref{tab:repair-mnist}.
The non-aggregated data is provided in \iftoggle{arxiv}{\Cref{tab:repair-mnist-non-aggregated}}{Sect. C of the supplementary material \cite{supplementary}}.
\tool successfully repairs all 100 instances, with a median test accuracy after repair of 96\%.
\Cref{tab:repair-mnist-instances} shows collective repair results, with \tool successfully repairing 10 counter-examples at once, and for three out of four instances, it repairs 25 counter-examples in one run.
nRepair repairs only 84 instances, with a slightly higher accuracy of 97.3\%, and better runtime.
Yet, it fails when tasked to collectively repair more than one counter-example in one execution.
DL2 only repairs ten instances, failing to repair 82, and timing out on eight. Also, it achieves the lowest median repair accuracy of 91.9\% for the successfully repaired cases. 
This indicates that DL2's counter-example generation cannot handle the DiffAI defense mechanism particularly well.
Also, its runtime is two orders of magnitude slower than \tool and nRepair (NR).

\begin{table}[t]
    \begin{center}
    \begin{tabularx}{\textwidth}{X cccc c c}
    \toprule
        & \multicolumn{4}{c}{\textbf{Repair Outcome}}&  \textbf{Accuracy \![\%]} & \textbf{Runtime [s]} \\
    \textbf{Tool} & \,\,\textsc{Success}\,\, & \,\,\textsc{Fail}\,\, & \,\,\textsc{Unknown}\,\, & \,\,\textsc{Timeout}             & median           & median \\
    \midrule
    \tool & 88           &   5 &   1 &   6 &  \textbf{69.6} & \textbf{4269.7} \\
    DL2   & \textbf{100} &   0 &   0 &   0 &  61.8          & 26724.1\, \\
    \bottomrule
    \end{tabularx}
    \end{center}
    \caption{\textbf{Robustness repair results of the CIFAR10 CNN}: \tool (this work), minimal modification (MM) \cite{DBLP:conf/lpar/GoldbergerKAK20}, and DL2 \cite{DBLP:conf/icml/FischerBDGZV19}.
    We compare the cumulative repair outcome for all 100 instances, the test accuracy (minimum, median, and maximum) after repair, and the median runtime.}
    \label{tab:repair-cifar10}
\end{table}

For the CIFAR10 CNN, we present the results in \Cref{tab:repair-cifar10}, with non-aggregated data given in \iftoggle{arxiv}{\Crefrange{tab:repair-cifar10-non-aggregated-a}{tab:repair-cifar10-non-aggregated-b}}{Sect. C of the supplementary material \cite{supplementary}}.
\tool and DL2 are successful for 88 and 100 instances, respectively.
\tool maintains the highest mean test accuracy after repair with a value of 69.6\%, with DL2 sacrificing quantity over accuracy, only achieving 61.6\%.
Also, DL2 is six times slower than \tool.
We observe a lower test accuracy than for the collision avoidance task and the MNIST network for all approaches.
These results suggest that the higher input dimension of CIFAR10 ($32\times32\times3$ versus $28\times28$ with MNIST) is a limiting factor not only for \tool but all the repair methods that we have evaluated.

\section{Discussion}

In \Cref{sec:overview} we gave an overview of our procedure.
Maintaining the correct functionality of the repaired DNN is a fundamental challenge:
a successful repair is worthless if we compromise accuracy for it.
We have demonstrated that \tool consistently achieves high performance on several types of networks, often outperforming state-of-the-art repair methods.
The performance of \tool for the collision avoidance and image classification tasks is the best among the methods compared in the evaluation in terms of successful repairs, accuracy, or scalability, while still providing formal safety guarantees. This demonstrates that \tool is highly suitable for safety-critical applications.

The quality and success of our repair technique stems from its algorithmic design. Instead of relying on accurate yet computationally expensive encodings backed by SMT or linear programming, we use heuristics based on global optimization to produce counter-examples fast.

The insights gained in our experiments also support the hypothesis that using a task loss function that integrates into standard DNN training procedures, as in our approach or in \cite{DBLP:conf/icml/FischerBDGZV19}, is not only more efficient, but also better in preserving the DNN's accuracy.
Although the approaches in \cite{DBLP:conf/lpar/GoldbergerKAK20,DBLP:conf/pldi/SotoudehT21} try to keep the modifications of DNN parameters minimal, original training data is not considered, and the experiments demonstrate that there may still be a significant negative impact on the model's test accuracy.

One limitation of our repair approach is that, for image classification tasks, it does not always return with a successfully repaired DNN within the specified time limit.
Future research may seek the combination of \tool with different penalty functions during re-training to gain insights into the quality of repair results when applied to other network architectures.

\section{Conclusion}

We presented \tool, an efficient technique for generating counter-examples and repairing deep neural networks (DNNs) such that they comply with a formal specification.
Due to its black-box nature, \tool supports arbitrary DNNs and specifications.
Our technique consists of two main components.
The first component (counter-example generation) translates the specification into an objective function, which becomes negative for all network inputs that violate the specification, and then detects counter-examples using a global optimization method.
The second component (repair) utilizes these counter-examples to make the DNN safe via penalized re-training.
\tool finally gives a safety guarantee for the resulting DNN using a verifier.
Experimental results demonstrate that \tool can be used effectively for both counter-example generation and repair of DNNs, generating useful counter-examples, achieving a high quality of repair, and outperforming existing approaches.

\subsection*{Acknowledgments}

This research was partly supported by DIREC - Digital Research Centre Denmark and the Villum Investigator Grant S4OS.

\bibliography{bibliography.bib}

\iftoggle{arxiv}{
    \newpage
    \appendix
    \numberwithin{table}{section}
    \numberwithin{figure}{section}
    \section{ACAS Xu Properties}\label{sec:acas-xu-properties}

An overview about the ACAS Xu input variables and output advisories is given in \Cref{tab:acasxu-overview}. 
The ten properties provided by \citet{DBLP:conf/cav/KatzBDJK17} are listed in \Cref{tab:experiments-acas-xu-specification}.

\begin{table}[H]
    \footnotesize
    \centering
    \begin{tabular}{@{}c p{5.7cm} c l @{}} \toprule
        Input $\vec{x}$\,\,  & Semantics                                                         & Output  $\vec{y}$ \,\,  & Semantics \\ \midrule
        $\rho$          & Distance from ownship to intruder                                 & $y_1$   & Clear-of-Conflict (COC) \\
        $\theta$        & Angle to intruder relative to ownship heading direction           &  $y_2$    & Weak left (WL)\\
        $\psi$          & Heading angle of intruder relative to ownship heading direction   & $y_3$    & Weak right (WR)  \\
        $v_\text{own}$  & Speed of ownship                                                  & $y_4$    & Strong left (SL) \\
        $v_\text{int}$  & Speed of intruder                                                 & $y_5$    & Strong right (SR) \\
        \bottomrule
    \end{tabular}
    \caption{\textbf{ACAS Xu network parameters.} Input variables and output classes for the ACAS Xu DNNs \cite{julian2016policy}.}
    \label{tab:acasxu-overview}
\end{table}
\begin{table}[H]
    \scriptsize
    \centering
    \begin{tabular}{@{}l p{1.35cm} c p{9.5cm}@{}}
        \toprule
        Spec & Model & \multicolumn{2}{c}{Definition} \\
        \midrule
        \(\varphi_1\) & all & 
        \(X_\phi\) & \([55947.691, \infty] \times \mathbb{R}^2 \times [1145, \infty] \times [-\infty, 60] \)  \\
        && \(Y_\phi\) &\(\left\lbrace \mathbf{y} \,|\, y_{1} \leq 1500 \right\rbrace\) \\[0.5em]
        
        \(\varphi_2\) & \(N_{2,1}\)--\(N_{5,9}\) & 
        \(X_\phi\) & \([55947.691, \infty] \times \mathbb{R}^2 \times [1145, \infty] \times [-\infty, 60]\) \\
        && \(Y_\phi\) & \(\left\lbrace \mathbf{y} \,|\, y_{1} \leq \max_{j \neq 1} y_j \right\rbrace\) \\[0.5em]
        
        \(\varphi_3\) & all except \(N_{1,7}\)--\(N_{1,9}\) &
        \(X_\phi\) & \([1500, 1800] \times [-0.06, 0.06] \times [3.10, \infty] \times [980, \infty] \times [960, \infty]\) \\
        && \(Y_\phi\) & \(\left\lbrace \mathbf{y} \,|\, y_{1} \geq \min_{j \neq 1} y_j \right\rbrace\) \\[0.5em]
        
        \(\varphi_4\) & all except \(N_{1,7}\)--\(N_{1,9}\) &
        \(X_\phi\) & \([1500, 1800] \times [-0.06, 0.06] \times [0, 0] \times [1000, \infty] \times [700, 800]\) \\
        && \(Y_\phi\) & \(\left\lbrace \mathbf{y} \,|\, y_{1} > \min_{j \neq 1} y_j \right\rbrace\) \\[0.5em]
        
        \(\varphi_5\) & \(N_{1,1}\) & 
        \(X_\phi\) & \([250, 400] \times [0.2, 0.4] \times [-3.141592, -3.141592 + 0.005] \times [100, 400] \times [0, 400]\) \\
        && \(Y_\phi\) & \(\left\lbrace \mathbf{y} \,|\, y_{5} \geq \min_{j \neq 5} y_j \right\rbrace\) \\[0.5em]
        
        \(\varphi_6\) & \(N_{1,1}\) &
        \(X_\phi\) & \([12000, 62000] \times [0.7, 3.141592] \cup [-3.141592, -0.7] \times [-3.141592, -3.141592 + 0.005] \times [100, 1200] \times [0, 1200]\)  \\
        && \(Y_\phi\) & \(\left\lbrace \mathbf{y} \,|\, y_{1} \geq \min_{j\neq 1} y_j \right\rbrace\) \\[0.5em]
        
        \(\varphi_7\) & \(N_{1,9}\) & 
        \(X_\phi\) & \([0, 60760] \times [-3.141592, 3.141592]^2 \times [100, 1200] \times [0, 1200]\)  \\
        && \(Y_\phi\) & \(\left\lbrace \mathbf{y} \,|\, \min_{j_1 \in \{4, 5\}}\mathbf{y}_{j_1} > \min_{j_2 \not\in \{4, 5\}} \mathbf{y}_{j_2}  \right\rbrace\) \\[0.5em]
        
        \(\varphi_8\) & \(N_{2,9}\) & 
        \(X_\phi\) & \([0, 60760] \times [-3.141592, -2.356194] \times [-0.1, 0.1] \times [600, 1200]^2 \)  \\
        && \(Y_\phi\) & \(\left\lbrace \mathbf{y} \,|\,\min_{j_1 \in \{1, 2\}}\mathbf{y}_{j_1} < \min_{j_2 \not\in \{1, 2\}} \mathbf{y}_{j_2}  \right\rbrace\) \\[0.5em]
        
        \(\varphi_9\) & \(N_{3,3}\) & 
        \(X_\phi\) & \([2000, 7000] \times [-0.4, -0.14] \times [-3.141592, -3.141592 + 0.01] \times [100, 150] \times [0, 150] \)  \\
        && \(Y_\phi\) & \(\left\lbrace \mathbf{y} \,|\, y_{4} \geq \min_{j \neq 4} y_j  \right\rbrace\) \\[0.5em]
        
        \(\varphi_{10}\) & \(N_{4,5}\) & 
        \(X_\phi\) & \([36000, 60760] \times [0.7, 3.141592] \times [-3.141592, -3.141592 + 0.01] \times [900, 1200] \times [600,1200] \)  \\
        && \(Y_\phi\) & \(\left\lbrace \mathbf{y} \,|\, y_{1} \geq \min_{j \neq 1} y_j  \right\rbrace\)\\
        \bottomrule
    \end{tabular}
    \caption{\textbf{ACAS Xu specifications} by \cite{DBLP:conf/cav/KatzBDJK17}. 
    The output indices are assumed to correspond to the following actions: \(y_1\): COC, \(y_2\): WL, \(y_3\): WR, \(y_4\): SL, \(y_5\): SR. 
    }
    \label{tab:experiments-acas-xu-specification}
\end{table}

\section{Choice of Global Optimization Method}

As a preliminary experiment, we evaluate three optimization methods: basin hopping \cite{DBLP:journals/advai/OlsonHMS12}, SHGO \cite{Endres18}, and differential evolution \cite{DBLP:journals/jgo/StornP97}, on property $\varphi_2$ and all 45 networks of the collision avoidance task. 
For this purpose, we compare runtime, the counter-example's objective value (the numerical output of the $\fsat$ function for the counter-example) which gives an estimate about its suitability for repair, and the number of function evaluations, shown in \Cref{fig:optimizers}.

\begin{figure}
    \centering
    \includegraphics[width=0.6\textwidth]{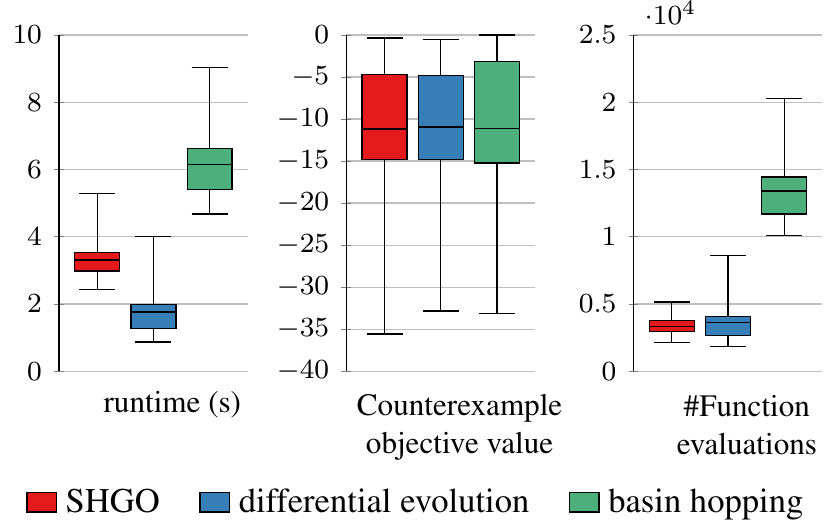}
    \caption{Comparison of optimization methods.}
    \label{fig:optimizers}
\end{figure}

\Cref{fig:optimizers} shows runtime, the objective value of the counter-examples found, and the number of function evaluations for state-of-the-art optimizers that can be used by \tool: basin hopping, SHGO, and differential evolution. 
The runtime of differential evolution is best.
The counter-example quality, measured as objective function value of the counter-example, where lower values are better than higher ones, is the lowest for SHGO.
A low objective value corresponds to a higher deviation from the safety specification and is important during the repair phase of \tool.
Also, SHGO needs few function evaluations, making it the most suitable to perform attacks on networks to which access is limited.
\newpage

\section{Non-Aggregated Data}

In this section, we give the non-aggregated data for the \textit{collision avoidance} and \textit{image classification} tasks from the experiments.
In the former, we repair 36 problem instances (consisting of a DNN and a safety property), and in the latter, we perform a robustness repair on a DNN for 100 distinct input images.
Also, we give detailed information of the runtime per algorithm phase: counter-example fixing, verification, counter-example generation, and overall.

\subsection{ACAS Xu}

This section gives additional information about the repair results for the collision avoidance task, counting the number of successful repairs, failures, unknown outcomes, and timeouts.

The runtimes per algorithm phase is given in in \Cref{tab:repair-acas-xu-overview-non-aggr}.

\begin{table}[H]
    \begin{center}
        \scriptsize
        \begin{tabularx}{\textwidth}{l @{\hskip 1em} cccc @{\hskip 1em} ccc c @{\hskip 1em} llll}
        \toprule
            & \multicolumn{4}{c}{\textbf{Rep. \!Outcome}} & \multicolumn{3}{c}{\textbf{Accuracy\! [\%]}} & \textbf{MAE} & \multicolumn{4}{c}{\textbf{Runtime \![s]}} \\
        \textbf{Tool} & \exsuccess & \exfailure & \,\,\exunknown\,\,\,\, & \extimeout\,\, & min & med & max & med &  Cx \!fixing & Verification & Cx \!generation & Overall \\
        \midrule
        \tool    & \textbf{28} &  0 &  2 &  6 & 14.5 & \textbf{99.5}   & 99.9 & \textbf{0.1}  & 13.0    & 264.7  & 6.5 & 299.6    \\ 
        DL2      &           3 & 33 &  0 &  0 & 89.5 & 87.6            & 96.0 & 1996.1        & 10658.5 & 94.9   & --  & 10690.6 \\ 
        MM       &           0 &  0 &  0 & 35 & 77.5 & 93.4            & 89.8& 6.03          & 8843.4  & 2528.2 & --  & 10834.1 \\ 
        NF       &          29 &  0 &  1 &  6 & \multicolumn{3}{c}{--} & --            & 0.1     & 9.9    & --  & \textbf{10.0}    \\
        \bottomrule
        \end{tabularx}
    \end{center}
    \caption{\textbf{Safety repair results of the ACAS Xu DNNs}: \tool (this work), DL2 \cite{DBLP:conf/icml/FischerBDGZV19}, nRepair (NR) \cite{DBLP:conf/qrs/DongSWWD21}, and minimal modification (MM) \cite{DBLP:conf/lpar/GoldbergerKAK20}.
    We compare the cumulative repair outcome for all 35 instances, median accuracy and mean average error (MAE) after repair, and the median runtime.
    \exsuccess{ }indicates successful repairs, \exfailure{ }indicates failed repairs, \extimeout{ }indicates a timeout, and `\exunknown' marks cases where the verifier returned \textsc{Unknown}.}
    \label{tab:repair-acas-xu-overview-non-aggr}
\end{table}

\newpage

\subsection{Diff-AI defended MNIST DNN}

\begin{table}[H]
    \begin{center}
    \scriptsize
        \begin{tabularx}{\textwidth}{l @{\hskip 1em} cccl ccl lllll}
        \toprule
        &      \multicolumn{4}{c}{\textbf{Rep. \!Outcome}}  & \multicolumn{3}{c}{\textbf{Accuracy \![\%]}} & \multicolumn{4}{c}{\textbf{Runtime \![s]}} \\
        \textbf{Tool} & \exsuccess & \exfailure & \,\,\exunknown\,\,\,\, & \extimeout\,\, \hspace{1em} & min           & med         & max \hspace{2em} &  Cx \!fixing & Verification & Cx \!generation & Overall  \\
        \midrule
        \tool & \textbf{100}    &   0 &   0 &   0 & 69.3 & 96.0          & 96.9 & 161.1   & 0.8  & 1.8 & 163.6    \\
        NR    & 84              &  16 &   0 &   0 & 59.9 & \textbf{97.3} & 97.9 & 6.9     & 26.9 & \,--  & \textbf{36.0}     \\
        DL2   & 10              &  82 &   0 &   8 & 77.3 & 91.9          & 96.6 & 18806.7 & 36.3 & \,--  & 18839.3 \\
        \bottomrule
        \end{tabularx}
    \end{center}
    \caption{\textbf{Robustness repair results of the DiffAI-defended \cite{DiffAi} MNIST DNN}: \tool (this work), nRepair (NR) \cite{DBLP:conf/qrs/DongSWWD21}, and DL2 \cite{DBLP:conf/icml/FischerBDGZV19}.
    We compare the cumulative repair outcome for all 100 instances, the test accuracy (minimum, median, and maximum) after repair, and the median runtime by algorithm phase.
    \exsuccess{ }indicates successful repairs, \exfailure{ }indicates failed repairs, \extimeout{ }indicates a timeout, and `\exunknown' marks cases where the verifier returned \textsc{Unknown}.}
        \label{tab:repair-mnist-runtime}
\end{table}

Verification (via ERAN) is a lot faster for \tool because it more frequently confirms that a given property is verified.
For DL2 and nRepair (NR), ERAN is (mis-)used to generate counter-examples, a task that ERAN is not designed for; it is designed to report satisfaction fast through the use of abstract interpretation. 

The results for DL2, as shown in \Cref{tab:repair-mnist-runtime}, are the best results obtained by testing five
different DL2 weight parameters: \(0.01, 0.05, 0.1, 0.2\) and \(0.5\).
Timeouts are ranked above failures, i.e.\ if repair with any of the weight values times out, that case is counted as a timeout in the above table. 

The runtime of one run of the DL2 tool is calculated as the sum of the runtime for the different weight values. 
Because of this the median runtime of the DL2 tool exceeds the timeout (10800) by more than two hours. 

\begin{table}[H]
\begin{center}
    \scriptsize
    \begin{tabularx}{\textwidth}{@{}l cc @{\hskip 1em} cccc XXX@{}}
    \toprule
           &   &  &  \multicolumn{4}{c}{\textbf{Repair Outcome}} & \multicolumn{3}{c}{\textbf{Accuracy \![\%]}} \\
    \textbf{Tool} & \textbf{Points} & \textbf{Instances}  & \,\,\textsc{Success}\,\, & \,\,\textsc{Fail}\,\, & \,\,\textsc{Unknown}\,\, & \,\,\textsc{Timeout}  & min  & med  & max \\
    \midrule
            & 1 & 100 & \textbf{100}  &   0 &   0 &   0 & 69.3 & 96.0 & 96.9  \\
    \tool   & 10 & 10 & \textbf{10}   &   0 &   0 &   0 & 89.9 & \textbf{93.1} & 94.2 \\
            & 25 & 4  & \textbf{3}    &   0 &   1 &   1 & 93.1 & \textbf{93.5} & 93.6 \\
    \midrule
            & 1  & 100 &  84 &  16 &   0 &   0 & 59.9 & \textbf{97.3} & 97.9 \\
    NR      & 10 & 10  &   0 &  10 &   0 &   0 & \multicolumn{3}{c}{ -- } \\
            & 25 & 4   &   0 &  4  &   0 &   0 & \multicolumn{3}{c}{ -- } \\
    \bottomrule
    \end{tabularx}
    \end{center}
    \caption{\textbf{Robustness repair results of the DiffAI-defended \cite{DiffAi} MNIST DNN}: \tool (this work), and nRepair (NR) \cite{DBLP:conf/qrs/DongSWWD21}.
    We compare the cumulative repair outcome for all 100 instances, with two partitions into groups of 10 and 25 instances, and the test accuracy (minimum, median, and maximum) after repair.}
    \label{tab:repair-mnist-instances-non-aggr}
\end{table}

\begin{table}
\begin{center}
\scriptsize
\begin{longtable}{@{}c ccccc | c ccccc@{}}
\toprule
      & \multicolumn{5}{c|}{\textbf{DL2 weight}} &  & \multicolumn{5}{c}{\textbf{DL2 weight}} \\
\textbf{Point} & 0.01 & 0.05 & 0.1 & 0.2 & 0.5 & \textbf{Point} & 0.01 & 0.05 & 0.1 & 0.2 & 0.5 \\
\midrule
0  &        \exfailure &        \exfailure &        \exfailure &        \exfailure &        \exfailure& 50 &        \exfailure &        \exfailure &        \exfailure &        \exfailure &        \exfailure\\
1  &        \exfailure &        \exfailure &        \exfailure &        \exfailure &        \exfailure& 51 &        \exfailure &        \exfailure &        \exfailure &        \exfailure &        \exfailure\\
2  &        \exfailure &        \exfailure &        \exfailure &        \exfailure &        \exfailure& 52 &        \exfailure &        \exfailure &        \exfailure &        \exfailure &        \exfailure\\
3  &        \exfailure &        \exfailure &        \exfailure &        \exfailure &        \exfailure& 53 &        \exfailure &        \exfailure &        \exfailure &        \exfailure &        \exfailure\\
4  &        \exfailure &        \exfailure &        \exfailure &        \exfailure &        \exfailure& 54 &        \exfailure &        \exfailure &        \exfailure &        \exfailure &        \exfailure\\
5  &        \exfailure &        \exfailure &        \exfailure &        \exfailure &        \exfailure& 55 &        \exfailure &        \exfailure &        \exfailure & \textbf{ 77.32\%} &        \exfailure\\
6  &        \exfailure &        \exfailure &        \exfailure &        \exfailure &        \exfailure& 56 &        \exfailure &        \exfailure &        \exfailure &        \exfailure &        \exfailure\\
7  &        \exfailure &        \exfailure &        \exfailure &        \exfailure &        \exfailure& 57 &        \exfailure &        \exfailure &        \exfailure &        \exfailure &        \exfailure\\
8  &        \exfailure &        \exfailure &        \exfailure &        \exfailure &        \exfailure& 58 &        \exfailure &        \exfailure &        \exfailure &        \exfailure &        \exfailure\\
9  &        \exfailure &        \exfailure &        \exfailure &        \exfailure &        \exfailure& 59 &        \exfailure &        \exfailure &        \exfailure &        \exfailure &        \exfailure\\
10 &        \exfailure &        \exfailure &        \exfailure &        \exfailure &        \exfailure& 60 &        \exfailure &        \exfailure &        \exfailure &        \exfailure &        \exfailure\\
11 &        \exfailure &        \exfailure &        \exfailure &        \exfailure &        \exfailure& 61 &        \exfailure &        \exfailure &        \exfailure & \textbf{ 91.58\%} &        \exfailure\\
12 &        \exfailure &        \exfailure &        \exfailure &        \exfailure &        \exfailure& 62 &        \exfailure &        \exfailure &        \exfailure &        \extimeout & \textbf{ 86.73\%}\\
13 &        \exfailure &        \exfailure &        \exfailure &        \exfailure &        \exfailure& 63 &        \exfailure &        \exfailure &        \exfailure &        \extimeout &        \exfailure\\
14 &        \exfailure &        \exfailure &        \exfailure &        \exfailure &        \exfailure& 64 &        \exfailure &        \exfailure &        \exfailure &        \exfailure &        \exfailure\\
15 &        \exfailure &        \exfailure &        \exfailure &        \exfailure &        \exfailure& 65 &        \exfailure &        \exfailure &        \exfailure &        \exfailure &        \exfailure\\
16 &        \exfailure &        \exfailure &        \exfailure &        \exfailure &        \exfailure& 66 &        \exfailure &        \exfailure &        \exfailure &        \exfailure &        \exfailure\\
17 &        \exfailure &        \exfailure &        \exfailure &        \exfailure &        \exfailure& 67 &        \exfailure &        \exfailure &        \exfailure &        \exfailure &        \exfailure\\
18 &        \exfailure &        \exfailure &        \exfailure &        \exfailure & \textbf{ 90.44\%}& 68 &        \exfailure &        \exfailure &        \exfailure &        \exfailure &        \exfailure\\
19 &        \exfailure &        \exfailure &        \exfailure &        \exfailure &        \exfailure& 69 &        \exfailure &        \exfailure &        \exfailure &        \exfailure &        \exfailure\\
20 &        \exfailure &        \exfailure &        \exfailure &        \exfailure & \textbf{ 87.04\%}& 70 &        \exfailure &        \exfailure &        \exfailure &        \exfailure &        \exfailure\\
21 &        \exfailure &        \exfailure &        \exfailure &        \exfailure &        \exfailure& 71 &        \exfailure &        \exfailure &        \exfailure &        \exfailure &        \exfailure\\
22 &        \exfailure &        \exfailure &        \exfailure &        \exfailure &        \exfailure& 72 &        \exfailure &        \exfailure &        \exfailure &        \exfailure &        \exfailure\\
23 &           91.10\% &        \exfailure &        \exfailure & \textbf{ 93.54\%} &        \exfailure& 73 &        \exfailure &        \exfailure &        \exfailure &        \exfailure &        \exfailure\\
24 &        \exfailure &        \exfailure &        \exfailure &        \exfailure &        \exfailure& 74 &        \exfailure & \textbf{ 96.58\%} &        \exfailure &        \exfailure &        \exfailure\\
25 &        \exfailure &        \exfailure &        \exfailure &        \exfailure &        \exfailure& 75 &        \exfailure &        \exfailure &        \exfailure &        \exfailure &        \exfailure\\
26 &        \exfailure &        \exfailure &        \exfailure &        \exfailure &        \exfailure& 76 &        \exfailure &        \exfailure &        \exfailure &        \exfailure &        \exfailure\\
27 &        \exfailure &        \exfailure &        \exfailure &        \exfailure &        \exfailure& 77 &        \exfailure &        \exfailure &        \exfailure & \textbf{ 92.16\%} &        \exfailure\\
28 &        \exfailure &        \exfailure &        \exfailure &        \exfailure &        \exfailure& 78 &        \exfailure &        \exfailure &        \exfailure & \textbf{ 96.51\%} &        \exfailure\\
29 &        \exfailure &        \exfailure &        \exfailure &        \exfailure &        \exfailure& 79 &        \exfailure &        \exfailure &        \exfailure &        \extimeout &        \exfailure\\
30 & \textbf{ 96.28\%} &        \exfailure &        \exfailure &           95.95\% &        \exfailure& 80 &        \exfailure &        \exfailure &        \exfailure &        \extimeout &        \exfailure\\
31 &        \exfailure &        \exfailure &        \exfailure &        \exfailure &        \exfailure& 81 &        \exfailure &        \exfailure &        \exfailure &        \extimeout &        \exfailure\\
32 &        \exfailure &        \exfailure &        \exfailure &        \exfailure &        \exfailure& 82 &        \exfailure &        \exfailure &        \exfailure &        \extimeout &        \exfailure\\
33 &        \exfailure &        \exfailure &        \exfailure &        \exfailure &        \exfailure& 83 &        \exfailure &        \exfailure &        \exfailure &        \extimeout &        \exfailure\\
34 &        \exfailure &        \exfailure &        \exfailure &        \exfailure &        \exfailure& 84 &        \exfailure &        \exfailure &        \exfailure &        \exfailure &        \exfailure\\
35 &        \exfailure &        \exfailure &        \exfailure &        \exfailure &        \exfailure& 85 &        \exfailure &        \exfailure &        \exfailure &        \exfailure &        \exfailure\\
36 &        \exfailure &        \exfailure &        \exfailure &        \exfailure &        \exfailure& 86 &        \exfailure &        \exfailure &        \exfailure &        \exfailure &        \exfailure\\
37 &        \exfailure &        \exfailure &        \exfailure &        \exfailure &        \exfailure& 87 &        \exfailure &        \exfailure &        \exfailure &        \exfailure &        \exfailure\\
38 &        \exfailure &        \exfailure &        \exfailure &        \exfailure &        \exfailure& 88 &        \exfailure &        \exfailure &        \exfailure &        \exfailure &        \exfailure\\
39 &        \exfailure &        \exfailure &        \exfailure &        \exfailure &        \exfailure& 89 &        \exfailure &        \exfailure &        \exfailure &        \exfailure &        \exfailure\\
40 &        \exfailure &        \exfailure &        \exfailure &        \exfailure &        \exfailure& 90 &        \exfailure &        \exfailure &        \exfailure &        \exfailure &        \exfailure\\
41 &        \exfailure &        \exfailure &        \exfailure &        \exfailure &        \exfailure& 91 &        \exfailure &        \exfailure &        \exfailure &        \extimeout &        \exfailure\\
42 &        \exfailure &        \exfailure &        \exfailure &        \exfailure &        \exfailure& 92 &        \exfailure &        \exfailure &        \exfailure &        \exfailure &        \exfailure\\
43 &        \exfailure &        \exfailure &        \exfailure &        \exfailure &        \exfailure& 93 &        \exfailure &        \exfailure &        \exfailure &        \exfailure &        \exfailure\\
44 &        \exfailure &        \exfailure &        \exfailure &        \exfailure &        \exfailure& 94 &        \exfailure &        \exfailure &        \exfailure &        \exfailure &        \exfailure\\
45 &        \exfailure &        \exfailure &        \exfailure &        \exfailure &        \exfailure& 95 &        \exfailure &        \exfailure &        \exfailure &        \exfailure &        \exfailure\\
46 &        \exfailure &        \exfailure &        \exfailure &        \exfailure &        \exfailure& 96 &        \exfailure &        \exfailure &        \exfailure &        \exfailure &        \exfailure\\
47 &        \exfailure &        \exfailure &        \exfailure &        \exfailure &        \exfailure& 97 &        \exfailure &        \exfailure &        \exfailure &        \exfailure &        \exfailure\\
48 &        \exfailure &        \exfailure &        \exfailure &        \exfailure &        \exfailure& 98 &        \exfailure &        \exfailure &        \exfailure &        \exfailure &        \exfailure\\
49 &        \exfailure &        \exfailure &        \exfailure &        \exfailure &        \exfailure& 99 &        \exfailure &        \exfailure &        \exfailure &        \exfailure &        \exfailure\\
\bottomrule
\end{longtable}
\end{center}
\caption{\textbf{Non-aggregated robustness repair results of the DiffAI-defended \cite{DiffAi} MNIST DNN}: DL2 \cite{DBLP:conf/icml/FischerBDGZV19}.
    Repair outcome for all 100 instances and the 5 DL2 weights, and median test accuracy.}
    \label{tab:repair-mnist-non-aggregated}
\end{table}

\newpage

\subsection{CIFAR10 CNN}

As for the MNIST experiment, the results for DL2 are the best results obtained by testing five
different DL2 weight parameters: \(0.05, 0.01, 0.1, 0.2\) and \(0.5\).
Refer to \Crefrange{tab:repair-cifar10-non-aggregated-a}{tab:repair-cifar10-non-aggregated-b}.

\begin{table}[H]
    \begin{center}
    \scriptsize
    \begin{tabularx}{\textwidth}{X cccc @{\qquad} XXX c}
    \toprule
        & \multicolumn{4}{c}{\textbf{Repair Outcome}}&  \multicolumn{3}{c}{\textbf{Accuracy [\%]}} & \textbf{Runtime [s]} \\
    \textbf{Tool} & \,\,\textsc{Success}\,\, & \,\,\textsc{Fail}\,\, & \,\,\textsc{Unknown}\,\, & \,\,\textsc{Timeout}  & min           & med         & max  & med \\
    \midrule
    \tool & 88 &   5 &   1 &   6 & 62.3 & \textbf{69.6} & 70.7 & \textbf{4269.7} \\
    DL2        & \textbf{100}          &   0 &   0 &  0 & 58.7 & 61.8          & 70.5 & 26724.1\, \\
    \bottomrule
    \end{tabularx}
    \end{center}
    \caption{\textbf{Robustness repair results of the CIFAR10 CNN}: \tool (this work), minimal modification (MM) \cite{DBLP:conf/lpar/GoldbergerKAK20}, and DL2 \cite{DBLP:conf/icml/FischerBDGZV19}.
    We compare the cumulative repair outcome for all 100 instances, the test accuracy (minimum, median, and maximum) after repair, and the median runtime.}
    \label{tab:repair-cifar10-non-aggr}
\end{table}

\begin{table}
    \begin{center}
    \scriptsize
    \begin{longtable}{@{}c ccccc@{}}
\toprule
      & \multicolumn{5}{c}{\textbf{DL2 weight}} \\
\textbf{Point} & 0.01 & 0.05 & 0.1 & 0.2 & 0.5 \\
\midrule
\endhead
0  & \(\textbf{64.69\%/82.70\%}\) &   62.36\%/75.78\% &   60.55\%/72.75\% &        \extimeout &   60.35\%/72.39\% \\
1  & \(\textbf{61.40\%/75.08\%}\) &   60.11\%/72.74\% &   60.30\%/72.19\% &        \extimeout &   59.48\%/71.17\% \\
2  &        \extimeout & \textbf{63.11\%/77.91\%} &   61.69\%/74.97\% &        \extimeout &   59.27\%/71.05\% \\
3  &        \extimeout & \textbf{61.84\%/75.63\%} &   59.75\%/71.68\% &        \extimeout &   58.79\%/70.16\% \\
4  & \(\textbf{58.65\%/70.04\%}\) &   57.55\%/67.97\% &   58.33\%/69.28\% &        \extimeout &   58.22\%/68.74\% \\
5  &        \extimeout & \textbf{63.48\%/78.53\%} &   62.59\%/77.12\% &        \extimeout &   61.79\%/74.94\% \\
6  &        \extimeout & \textbf{62.82\%/76.65\%} &   60.58\%/72.79\% &        \extimeout &   59.99\%/71.98\% \\
7  &        \extimeout & \(\textbf{61.19\%/73.58\%}\) &   59.66\%/71.55\% &        \extimeout &   59.17\%/70.59\% \\
8  &        \extimeout & \(\textbf{70.51\%/94.99\%}\) &   68.36\%/89.87\% &        \extimeout &   66.57\%/86.32\% \\
9  &        \extimeout & \(\textbf{63.93\%/79.41\%}\) &   62.00\%/75.42\% &        \extimeout &   61.74\%/74.83\% \\
10 &        \extimeout &        \extimeout &   59.43\%/71.20\% &        \extimeout & \textbf{59.65\%/71.58\%} \\
11 &        \extimeout & \(  59.55\%/70.96\%\) & \textbf{59.90\%/71.04\%} &        \extimeout &   59.23\%/70.29\% \\
12 & \(\textbf{61.11\%/73.93\%}\) & \(  60.52\%/72.53\%\) &   58.77\%/70.29\% &        \extimeout &   59.37\%/70.24\% \\
13 &        \extimeout & \(\textbf{61.16\%/73.29\%}\) &   58.97\%/70.50\% &        \extimeout &   58.48\%/69.92\% \\
14 & \(\textbf{61.77\%/74.57\%}\) &        \extimeout &   59.90\%/71.64\% &   59.56\%/71.06\% &   59.06\%/70.13\% \\
15 &        \extimeout & \(\textbf{60.21\%/71.72\%}\) &   59.37\%/70.40\% &   59.74\%/70.48\% &   59.32\%/70.32\% \\
16 & \(\textbf{63.42\%/78.45\%}\) & \(  63.17\%/77.44\%\) &   61.09\%/73.95\% &   60.65\%/72.89\% &   60.17\%/71.84\% \\
17 &        \extimeout &        \extimeout & \textbf{61.10\%/73.76\%} &   59.51\%/71.34\% &   59.72\%/71.49\% \\
18 &        \extimeout & \(\textbf{60.63\%/73.01\%}\) &   58.85\%/70.38\% &   59.15\%/70.30\% &   59.03\%/70.19\% \\
19 &        \extimeout &        \extimeout &   59.37\%/71.25\% &   59.77\%/71.73\% & \textbf{59.79\%/71.50\%} \\
20 & \(\textbf{64.23\%/80.82\%}\) & \(  63.39\%/77.95\%\) &   61.07\%/73.70\% &   60.98\%/73.67\% &   61.13\%/72.98\% \\
21 & \(\textbf{64.81\%/81.27\%}\) & \(  62.64\%/76.13\%\) &   60.94\%/73.60\% &   60.26\%/72.21\% &   59.94\%/71.47\% \\
22 & \(\textbf{61.18\%/73.50\%}\) &        \extimeout &   59.55\%/70.69\% &   58.98\%/69.91\% &   59.27\%/70.53\% \\
23 & \(\textbf{66.57\%/86.14\%}\) & \(  65.26\%/82.39\%\) &   63.37\%/78.26\% &   63.25\%/77.02\% &   60.67\%/72.48\% \\
24 & \(\textbf{63.03\%/78.14\%}\) & \(  62.04\%/75.37\%\) &   59.23\%/70.73\% &   59.11\%/70.35\% &   59.15\%/70.20\% \\
25 &        \extimeout & \(\textbf{63.66\%/78.67\%}\) &   61.78\%/74.78\% &        \extimeout &   62.60\%/76.19\% \\
26 & \(  60.91\%/74.48\%\) & \(\textbf{60.93\%/73.32\%}\) &   58.30\%/69.66\% &        \extimeout &   57.55\%/68.26\% \\
27 &        \extimeout &        \extimeout &   61.72\%/74.29\% &        \extimeout & \textbf{62.16\%/75.23\%} \\
28 &        \extimeout & \(\textbf{61.49\%/74.43\%}\) &   59.92\%/71.58\% &        \extimeout &   59.93\%/70.93\% \\
29 & \(\textbf{64.57\%/80.19\%}\) &        \extimeout &   60.97\%/73.36\% &        \extimeout &   59.64\%/71.04\% \\
30 & \(\textbf{60.68\%/73.05\%}\) &        \extimeout &   58.45\%/69.65\% &        \extimeout &   58.52\%/69.52\% \\
31 &        \extimeout & \(  61.57\%/74.67\%\) & \textbf{63.30\%/78.02\%} &        \extimeout &   61.75\%/75.17\% \\
32 & \(\textbf{61.23\%/74.45\%}\) & \(  60.88\%/72.76\%\) &   59.12\%/71.08\% &        \extimeout &   59.23\%/70.89\% \\
33 & \(\textbf{64.51\%/81.20\%}\) &        \extimeout &   60.72\%/73.61\% &        \extimeout &   60.10\%/72.25\% \\
34 & \(\textbf{61.64\%/74.57\%}\) &        \extimeout &   59.45\%/70.36\% &        \extimeout &   59.38\%/70.20\% \\
35 & \(\textbf{65.00\%/82.70\%}\) &        \extimeout &   61.47\%/75.12\% &        \extimeout &   61.25\%/74.47\% \\
36 & \(\textbf{61.67\%/75.04\%}\) &        \extimeout &   59.25\%/70.85\% &        \extimeout &   58.99\%/70.36\% \\
37 & \(\textbf{63.99\%/80.04\%}\) & \(  61.47\%/74.74\%\) &   59.81\%/71.12\% &        \extimeout &   59.50\%/70.77\% \\
38 & \(\textbf{60.06\%/71.35\%}\) &   58.90\%/70.05\% &   58.94\%/70.00\% &   59.14\%/70.31\% &   58.55\%/69.32\% \\
39 &        \extimeout & \textbf{69.27\%/91.76\%} &   63.44\%/78.81\% &   62.41\%/77.02\% &   60.57\%/73.64\% \\
40 & \(\textbf{61.54\%/74.86\%}\) &   60.70\%/73.22\% &   60.36\%/72.80\% &   60.38\%/72.11\% &   59.98\%/71.59\% \\
41 & \(\textbf{63.76\%/79.21\%}\) &   63.09\%/77.81\% &   62.19\%/75.98\% &   60.94\%/72.88\% &   60.70\%/72.47\% \\
42 & \(\textbf{64.28\%/80.08\%}\) &   62.49\%/75.75\% &   62.31\%/75.39\% &   62.24\%/74.97\% &   62.10\%/74.72\% \\
43 &        \extimeout & \textbf{60.29\%/72.36\%} &   59.99\%/71.96\% &   59.70\%/71.52\% &   59.78\%/71.41\% \\
44 & \(\textbf{61.62\%/74.53\%}\) &   59.98\%/72.03\% &   59.57\%/71.47\% &   59.54\%/71.13\% &   59.22\%/70.63\% \\
45 & \(  63.65\%/79.37\%\) &   64.15\%/80.65\% &   68.34\%/89.16\% & \textbf{69.24\%/91.76\%} &   67.60\%/86.57\% \\
46 & \(\textbf{61.52\%/74.52\%}\) &   59.40\%/71.18\% &   59.16\%/70.85\% &   59.09\%/70.67\% &   59.10\%/70.74\% \\
47 &        \extimeout & \textbf{62.52\%/76.27\%} &   62.19\%/76.05\% &   61.94\%/75.38\% &   61.10\%/73.72\% \\
48 & \(\textbf{61.78\%/74.71\%}\) &   60.49\%/71.54\% &   60.49\%/71.67\% &   60.35\%/71.23\% &   60.28\%/71.32\% \\
49 & \(\textbf{62.72\%/77.17\%}\) &   60.59\%/72.73\% &   59.78\%/71.48\% &   59.89\%/71.49\% &   59.20\%/70.83\% \\
\bottomrule
\end{longtable}

    \end{center}
    \caption{\textbf{Non-aggregated robustness repair results of the CIFAR10 CNN}: DL2 \cite{DBLP:conf/icml/FischerBDGZV19}.
    Repair outcome for the instances 0 to 49 and the 5 DL2 weights, and median test accuracy.}
    \label{tab:repair-cifar10-non-aggregated-a}
\end{table}

\begin{table}
    \begin{center}
    \scriptsize
    \begin{longtable}{@{}c ccccc@{}}
\toprule
      & \multicolumn{5}{c}{\textbf{DL2 weight}} \\
\textbf{Point} & 0.01 & 0.05 & 0.1 & 0.2 & 0.5 \\
\midrule
\endhead
50 &        \extimeout & \textbf{63.79\%/79.83\%} &   61.59\%/74.98\% &        \extimeout &   61.47\%/74.59\% \\
51 & \(\textbf{65.39\%/84.18\%}\) &   63.12\%/78.54\% &   62.35\%/76.69\% &        \extimeout &   59.23\%/70.78\% \\
52 & \(\textbf{63.71\%/78.90\%}\) &   62.58\%/75.52\% &   62.48\%/75.31\% &        \extimeout &   62.00\%/74.73\% \\
53 & \(\textbf{59.89\%/71.18\%}\) &   58.81\%/69.71\% &   58.57\%/69.53\% &        \extimeout &   58.43\%/69.49\% \\
54 &        \extimeout & \textbf{60.38\%/73.21\%} &   60.33\%/72.47\% &        \extimeout &   59.64\%/71.25\% \\
55 &        \extimeout &   59.58\%/71.65\% & \textbf{59.90\%/71.74\%} &        \extimeout &   59.27\%/70.88\% \\
56 & \(\textbf{62.26\%/76.07\%}\) &   60.55\%/72.94\% &   59.93\%/71.53\% &        \extimeout &   59.18\%/70.14\% \\
57 & \(\textbf{60.70\%/72.99\%}\) &   59.41\%/70.49\% &   59.09\%/70.16\% &        \extimeout &   59.01\%/70.06\% \\
58 & \(\textbf{64.54\%/80.08\%}\) &   62.25\%/75.91\% &   61.67\%/74.63\% &        \extimeout &   59.12\%/70.04\% \\
59 & \(\textbf{66.00\%/84.16\%}\) &   64.14\%/79.80\% &   63.07\%/77.29\% &        \extimeout &   61.00\%/73.39\% \\
60 & \(\textbf{61.91\%/74.97\%}\) &   60.19\%/71.53\% &   60.05\%/71.54\% &   60.08\%/71.47\% &   60.07\%/71.48\% \\
61 &        \extimeout &   60.64\%/72.81\% & \textbf{61.01\%/72.99\%} &   59.91\%/71.82\% &   60.00\%/71.51\% \\
62 & \(\textbf{61.92\%/76.44\%}\) &   61.61\%/74.82\% &   61.35\%/74.55\% &   61.12\%/73.83\% &   60.17\%/71.77\% \\
63 & \(\textbf{60.68\%/73.61\%}\) &   60.25\%/72.10\% &   59.77\%/71.45\% &   59.63\%/71.13\% &   59.41\%/71.02\% \\
64 &        \extimeout & \textbf{59.61\%/70.67\%} &   59.01\%/70.15\% &   59.32\%/70.66\% &   59.34\%/70.68\% \\
65 & \(\textbf{60.52\%/72.58\%}\) &   58.65\%/69.98\% &   59.09\%/70.15\% &   58.25\%/68.82\% &   58.31\%/69.10\% \\
66 & \(\textbf{64.37\%/80.68\%}\) &   60.47\%/73.23\% &   60.84\%/73.17\% &   60.90\%/72.96\% &   60.44\%/72.15\% \\
67 &        \extimeout & \textbf{63.63\%/78.36\%} &   62.91\%/76.85\% &   61.98\%/75.49\% &   60.62\%/72.70\% \\
68 &        \extimeout &   59.26\%/70.57\% &   59.17\%/70.53\% &   59.24\%/70.53\% & \textbf{59.30\%/70.55\%} \\
69 & \(\textbf{61.42\%/74.71\%}\) &   60.25\%/72.16\% &   59.72\%/71.18\% &   59.33\%/70.93\% &   59.12\%/70.26\% \\
70 & \(\textbf{66.34\%/85.89\%}\) &   63.98\%/80.59\% &   63.54\%/78.94\% &   62.39\%/76.93\% &   60.22\%/72.61\% \\
71 & \(\textbf{66.84\%/86.42\%}\) &   66.09\%/85.21\% &   65.46\%/83.99\% &   64.45\%/82.04\% &   63.55\%/79.60\% \\
72 & \(  67.01\%/86.71\%\) &   66.67\%/86.02\% & \textbf{67.02\%/86.24\%} &   65.94\%/85.16\% &   66.83\%/86.01\% \\
73 & \(\textbf{59.66\%/71.01\%}\) &   59.65\%/71.33\% &   59.24\%/70.76\% &   58.86\%/70.26\% &   59.08\%/70.34\% \\
74 &        \extimeout & \textbf{63.85\%/79.45\%} &   63.76\%/79.97\% &   63.05\%/77.28\% &   62.57\%/75.93\% \\
75 & \(\textbf{61.02\%/74.60\%}\) &   60.04\%/71.88\% &   59.71\%/71.46\% &        \extimeout &   59.93\%/71.46\% \\
76 & \(\textbf{63.73\%/78.18\%}\) &   61.39\%/74.65\% &   61.13\%/74.04\% &        \extimeout &   59.13\%/69.94\% \\
77 &        \extimeout & \textbf{59.83\%/71.39\%} &   59.67\%/70.76\% &        \extimeout &   59.53\%/70.40\% \\
78 & \(\textbf{63.31\%/78.70\%}\) &   61.92\%/75.31\% &   61.15\%/73.74\% &        \extimeout &   59.30\%/70.47\% \\
79 & \(\textbf{59.55\%/71.52\%}\) &   58.49\%/70.20\% &   58.36\%/69.67\% &        \extimeout &   58.66\%/69.75\% \\
80 & \(\textbf{62.44\%/76.35\%}\) &   60.29\%/71.64\% &   60.39\%/71.91\% &        \extimeout &   60.30\%/71.74\% \\
81 & \(\textbf{61.52\%/74.36\%}\) &   59.16\%/70.75\% &   59.00\%/70.44\% &   58.82\%/70.04\% &   58.43\%/69.83\% \\
82 & \(\textbf{61.27\%/74.04\%}\) &   59.65\%/71.26\% &   59.67\%/71.21\% &   59.60\%/71.22\% &   59.55\%/71.19\% \\
83 &        \extimeout & \textbf{61.81\%/75.05\%} &   61.76\%/74.85\% &   61.51\%/74.16\% &   60.66\%/72.83\% \\
84 &        \extimeout & \textbf{63.59\%/79.35\%} &   61.78\%/75.78\% &   61.41\%/75.38\% &   61.09\%/73.79\% \\
85 &        \extimeout & \textbf{59.67\%/71.34\%} &   59.46\%/71.04\% &   58.81\%/70.16\% &   58.71\%/69.98\% \\
86 &        \extimeout & \textbf{62.21\%/76.76\%} &   60.78\%/73.47\% &   60.71\%/74.70\% &   59.65\%/71.59\% \\
87 & \(\textbf{61.11\%/73.62\%}\) &   60.17\%/71.73\% &   59.58\%/70.98\% &   59.70\%/70.90\% &   59.40\%/70.52\% \\
88 &        \extimeout &   59.99\%/72.10\% & \textbf{60.42\%/72.34\%} &   60.34\%/72.15\% &   60.08\%/71.94\% \\
89 & \(\textbf{62.48\%/76.14\%}\) &   60.37\%/72.41\% &   60.43\%/72.36\% &   60.25\%/72.25\% &   60.23\%/72.01\% \\
90 & \(\textbf{61.46\%/74.25\%}\) &   60.66\%/73.09\% &   60.23\%/72.24\% &   59.79\%/71.60\% &   59.01\%/70.33\% \\
91 & \(\textbf{65.12\%/82.80\%}\) &   62.85\%/77.83\% &   61.89\%/75.58\% &   61.51\%/74.42\% &   61.39\%/74.06\% \\
92 &        \extimeout &   59.34\%/71.03\% &   59.44\%/70.89\% &   59.41\%/70.89\% & \textbf{59.44\%/70.85\%} \\
93 & \(\textbf{61.84\%/74.95\%}\) &   60.47\%/72.73\% &   59.98\%/71.92\% &   59.65\%/71.40\% &   58.88\%/70.28\% \\
94 & \(\textbf{62.97\%/77.87\%}\) &   61.66\%/74.64\% &   61.15\%/73.98\% &   61.08\%/73.64\% &   61.08\%/72.88\% \\
95 & \(\textbf{61.45\%/74.66\%}\) &   61.22\%/73.52\% &   60.47\%/72.23\% &   60.31\%/71.02\% &   60.08\%/71.60\% \\
96 & \(\textbf{61.28\%/74.01\%}\) &   60.53\%/72.46\% &   59.28\%/70.92\% &   59.15\%/70.52\% &   59.05\%/70.27\% \\
97 &        \extimeout &   60.80\%/72.65\% &   60.84\%/72.77\% &   60.88\%/72.82\% & \textbf{61.00\%/72.89\%} \\
98 &        \extimeout & \textbf{60.34\%/72.61\%} &   60.27\%/71.85\% &   59.60\%/70.75\% &   60.14\%/72.01\% \\
99 & \(\textbf{63.58\%/79.34\%}\) &   60.85\%/73.33\% &   60.63\%/72.59\% &   60.23\%/72.35\% &   59.98\%/71.27\% \\
\bottomrule
\end{longtable}

    \end{center}
    \caption{\textbf{Non-aggregated robustness repair results of the CIFAR10 CNN}: DL2 \cite{DBLP:conf/icml/FischerBDGZV19}.
    Repair outcome for the instances 50 to 99 and the 5 DL2 weights, and median test accuracy.}
    \label{tab:repair-cifar10-non-aggregated-b}
\end{table}
}{}

\end{document}


\maketitle

\section{General Changes}

Since the submission to the NeurIPS conference, we further developed our method and included these developments into the final version of our paper.
In essence, while the original only discussed the \emph{attack} (falsification) capabilities of our method, we have now also included \emph{repair}.
This is also reflected in new experiments.

\section{Changes concerning original reviews}

We believe that we addressed all statements of the individual reviews. 
Detailed discussions are given hereafter.

\subsection{Review iEae}

\textit{This paper proposes a technique for falsification based verification of DNNs. [...] Unfortunately, the algorithm provides only one-sided guarantee.}

This is a misunderstanding. We only propose a falsification approach. Verification is not our goal. The confusion may arise because we also compare to verification approaches; we do this because some of these approaches can also find counter-examples.

Unlike verification approaches, dedicated falsification approaches are typically faster in counter-example discovery, while verification approaches are typically optimized for the verification case. Having no verification capabilities is thus not a limitation in practice. If interested in a both-sided answer, it is usually best to use falsification approaches like ours concurrently with verification approaches. This is now included in our repair process, which first performs several steps of falsification and as a final step, verification.

\textit{The idea of finding counter-examples to formal properties of DNNs using objective function optimization is not new.}

Indeed, optimization-based approaches need to define an objective function to be optimized. In that sense our approach is similar to many other approaches. Our concrete instantiation of the objective function has (to the best of our knowledge) not been proposed before. In the evaluation we have compared to other state-of-the-art approaches such as DL2 and shown that our idea outperforms them.

\textit{Also, the literature contains critiques of the SHGO technique as well. The authors seem to project SHGO as the panacea for global function optimization -- this is a bit misleading in my opinion, and the authors should have addressed the limitations of SHGO (and hence limitations of their technique) as well.}

When we experimented with different optimizers, we found that SHGO worked best in our experiments.
But we agree that it is not the universally best optimizer.
We have now added an experiment that clarifies this point by comparing runtime, number of function evaluations, and quality of the counter-example used for the repair step.

\subsection{Review 9YL2}

\textit{There is no comparison against gradient-based methods for counter-example generation.}
Gradient-based methods such as FGSM have been shown to be outperformed by HopSkipJumpAttack (HSJA). We have compared against HSJA.

\textit{Several SoA verification tools are not included in the experimental comparison (e.g., [1,2,3]).}

We have discussed this in the related work section of the introduction.

\textit{DeepGO, which is also based on global optimisation, is not included in the MNIST and CIFAR10 experiments.}

HSJA and DL2 are also based on optimization, and are included in the MNIST and CIFAR10 experiments.

\textit{In particular, it is not clear why SHGO is used only for property falsification. Could it not in principle be used for certification?}

As discussed, SHGO as well as the other optimizers do not offer any guarantee to find the minimum, and thus cannot perform certification.

[1] Bak et. al, Improved Geometric Path Enumeration for Verifying ReLU Neural Networks, CAV'20.

[2] G. Katz, et al. The Marabou Framework for Verification and Analysis of Deep Neural Networks, CAV19.

[3] E. Botoeva, P. Kouvaros, J.~Kronqvist, A. Lomuscio, R. Misener, Efficient Verification of Neural Networks via Dependency Analysis, AAAI20.

\subsection{Review VTJE}

\textit{Have you tried other optimization algorithms than SHGO?}
In our experiments, we now compare basin-hopping and differential evolution, but due to better scalability and higher quality of counter-examples, we finally settled on SHGO.

\textit{A high level question: In this work we care about finding counter-examples, but is there any notion of a quality of the counter-examples found? Assuming that downstream task is to actually improve neural network performance by making use of counter-examples.}

For our repair procedure, we now discuss quality of the counter-example w.r.t. its objective value. The lower this value, the higher quality the counter-example, as this means a high deviation from the specification.
This is also evaluated.

\textit{(a) intuition on when they think their approach would perform sub-optimally}
The main limitation is in the optimizer. We found that SHGO has problems in case of objective-function surfaces where the gradient has an abrupt decrease in a very small area. We managed to mitigate this problem for some cases using our sampling strategy, but we cannot guarantee that this strategy covers all cases of this nature.

\textit{(b) what happens if there is no counter-example and how users can be sure that counter-example does not exist}
Without a counter-example, the repair method immediately invokes the verifier, which than either guarantees non-existence of such a counter-example or returns one. Interestingly, ERAN, the verifier we use, times out for some of the harder cases, ongoing experiments show the same for Marabou.

\textit{(c) is there a notion of quality of counter-examples if we wanted to retrain neural network to perform correctly in as few steps as possible.}

This is now discussed and included in our experiments.

\subsection{Review c3Eq}

\textit{As far as I understand, there must be a clear trade-off here: either you have to make a (n exponentially) large number of DNN executions (the proposed approach), or you suffer from some formal encoding and dealing with an NP-hard problem with the use of some form of automated reasoning (a number of SOTA approaches).}

Our new experiments concerning repair now show that executing the attack and only after all repair steps have been completed, executing the formal verifier, is a good trade-off.

\textit{Finally, the main issue pertaining to this approach is that, unless I am mistaken, it cannot provide any formal proof of the systems analyzed being correct (i.e. satisfying the target properties), as there is no formal reasoning about the system. What I mean is that if after a number of iterations there is no counter-example found, the approach is only able to report UNKNOWN.}

Indeed, our approach is a falsification approach. This criticism is valid for any other state-of-the art falsification approach. Returning the answer that "no counter-example exists" would be equivalent to verification. As the reviewer remarked, verification approaches solve an NP-hard problem and are thus inherently non-scalable. We believe that dedicated falsification approaches such as ours present a valid solution to these scalability problems and thus do not see this as a drawback but rather as a feature.

Unlike verification approaches, dedicated falsification approaches are typically faster in counter-example discovery, while verification approaches are typically optimized for the verification case.

This is demonstrated in our new experiments concerning repair of neural networks.